%% file: main_v1.tex
\pdfoutput=1
\documentclass[letterpaper]{article}

\usepackage{microtype}
\usepackage{graphicx}
\usepackage{booktabs} 
\usepackage{adjustbox}
\usepackage{float}
\usepackage{hyperref}
\usepackage[dvipsnames]{xcolor}

\usepackage[accepted]{icml2023}

\usepackage{amsmath}
\usepackage{amsfonts}
\usepackage{amssymb}
\usepackage{mathtools}
\usepackage{amsthm}
\usepackage{mathrsfs}
\usepackage{xfrac}

\usepackage[capitalize,noabbrev]{cleveref}
\usepackage[inline]{enumitem}
\usepackage{pifont}
\newcommand{\cmark}{{\color{ForestGreen}{\ding{51}}}}%
\newcommand{\xmark}{\ding{55}}%
\theoremstyle{plain}
\newtheorem{theorem}{Theorem}[section]
\newtheorem{proposition}[theorem]{Proposition}
\newtheorem{lemma}[theorem]{Lemma}

\theoremstyle{definition}
\newtheorem{definition}[theorem]{Definition}
\newtheorem{assumption}[theorem]{Assumption}
\theoremstyle{remark}

\input{shortcuts}
\icmltitlerunning{MAHALO: Unifying Offline RL and IL from Observations}

\begin{document}

\twocolumn[
\icmltitle{MAHALO: Unifying Offline Reinforcement Learning and \\
           Imitation Learning from Observations}



\icmlsetsymbol{equal}{*}

\begin{icmlauthorlist}
\icmlauthor{Anqi Li}{uw}
\icmlauthor{Byron Boots}{uw}
\icmlauthor{Ching-An Cheng}{msr}
\end{icmlauthorlist}

\icmlaffiliation{uw}{University of Washington}
\icmlaffiliation{msr}{Microsoft Research}

\icmlcorrespondingauthor{Anqi Li}{anqil4@cs.washington.edu}

\icmlkeywords{Offline Reinforcement Learning, Offline Imitation Learning, Imitation Learning from Observations}

\vskip 0.3in
]



\printAffiliationsAndNotice{} 

\begin{abstract}

We study a new paradigm for sequential decision making, called offline policy learning from observations (PLfO). Offline PLfO aims to learn policies using datasets with substandard qualities:
\begin{enumerate*}[label=\emph{\arabic*)}]
    \item only a subset of trajectories is labeled with rewards,
    \item labeled trajectories may not contain actions,
    \item labeled trajectories may not be of high quality, and 
    \item the data may not have full coverage.
\end{enumerate*}
Such imperfection is common in real-world learning scenarios, and offline PLfO encompasses many existing offline learning setups, including offline imitation learning (IL), offline IL from observations (ILfO), and offline reinforcement learning (RL).
In this work, we present a generic approach to offline PLfO, called \underline{M}odality-agnostic \underline{A}dversarial \underline{H}ypothesis \underline{A}daptation for \underline{L}earning from \underline{O}bservations (MAHALO).
Built upon the pessimism concept in offline RL, MAHALO optimizes the policy using a performance lower bound that accounts for uncertainty due to the dataset's insufficient coverage.
We implement this idea by adversarially training data-consistent critic and reward functions, 
which forces the learned policy to be robust to data deficiency.
We show that MAHALO consistently outperforms or matches specialized algorithms across a variety of offline PLfO tasks  in theory and experiments. Our code is available at \href{https://github.com/AnqiLi/mahalo}{https://github.com/AnqiLi/mahalo}.

\end{abstract}

\vspace{-5mm}
\section{Introduction}
\vspace{-1mm}

\input{tables/formulation.tex}

Online reinforcement learning (RL) has shown great promise 
in solving simulated tasks \cite{silver2016mastering,mnih2015human}. 
However, exploratory interactions with the environment, which are central to online RL, often can not be afforded in risk-sensitive 
applications, such as robotics~\cite{ibarz2021train} and healthcare~\cite{gottesman2018evaluating}. In these domains, it is more practical to consider an offline setting~\cite{levine2020offline}, where data is collected by behavioral policies satisfying certain criteria.

There are two main approaches to solving decision making problems offline: offline imitation learning (IL)~\cite{chang2021mitigating,kidambi2021mobile,kim2021demodice} and offline RL~\cite{fujimoto2019off,levine2020offline}. 
Offline IL generally does not assume access to the reward. 
Theses approaches learn with a small set of expert demonstrations and potentially a separate dynamics dataset with unknown quality. 
Offline IL seeks to mimic expert behavior while avoiding distribution shift caused by using offline datasets. 
Offline imitation learning from observations (ILfO)~\cite{kidambi2021mobile,ma2022versatile}
further relaxes the requirements of expert actions. ILfO allows learning from experts with different action spaces~\cite{edwards2020estimating}, or when the expert has a different action modality or a embodiment~\cite{cao2021learning,radosavovic2021state}.


%
Offline RL, on the other hand, does not require expert-level demonstrations. It instead assumes that each transition in the offline dataset is labeled with reward. 
The goal of offline RL is to learn a policy which
\begin{enumerate*}[label=\emph{\arabic*)}]
    \item always improves upon the behavioral policy~\cite{fujimoto2019off,laroche2019safe}, and 
    \item can outperform any other policies whose state-action distribution is covered by data~\cite{xie2021bellman}.
\end{enumerate*}

However, in real-world applications, it is expensive to either acquire expert-level demonstrations (even if they are observation-only), or label every transition with reward. In this paper, we propose a more general and realistic formulation called offline \emph{Policy Learning from Observations} (PLfO). Our goal is to learn from datasets where
\begin{enumerate*}[label=\emph{\arabic*)}]
    \item a subset of trajectories is labeled with rewards,
    \item labeled trajectories may not contain actions,
    \item labeled trajectories may not be of high quality, and
    \item the overall data may not have full coverage. 
\end{enumerate*}
The flexibility of this formulation allows us to directly take advantage of more data sources, such as dynamics data collected for other tasks and reward data collected by a non-expert agent with a different action space.

Offline PLfO considers two offline datasets:  the reward dataset $\DD_R=\{(s,r,s')\}$ and dynamics dataset $\DD_\AA=\{(s,a,s')\}$, 
\rev{where the dynamics dataset is consistent with the Markovian dynamics that the learner aims to solve (i.e. it is collected by agents that have the same embodiment as the learner).
}
In offline RL setting, the reward and dynamics datasets are aligned, since they are from the same underlying dataset $\DD = \{(s,a,r,s')\}$. Recent work~\cite{yu2022leverage} relaxes this requirement by assuming that only a subset of transitions are labeled with rewards, i.e., the set of state transitions contained in $\DD_R$ is a subset of $\DD_\AA$. On the contrary, in offline PLfO,  \rev{we {make no assumption} on how these two datasets are related to each other. }

Offline ILfO can also be viewed as a special case of offline PLfO. Although ILfO does not assume knowledge of reward, it makes an implicit assumption that expert trajectories attain high returns, 
while making no assumptions on reward information elsewhere (e.g., on the dynamics data $\DD_A$). In other words, from the perspective of offline PLfO, {expert demonstrations essentially act as the reward-labeled dataset $\DD_R$} for ILfO. 
Practically, we can simply label the expert demonstrations with the maximum reward. This observation is in line with existing work
~\cite{fu2018learning,eysenbach2021replacing,smith2023strong}. We refer readers to~\cref{tab:formulation} for a summary of comparison between offline PLfO and existing formulations. In~\cref{sec:related-work} we provide a more comprehensive literature review.

The key challenge to offline PLfO is the mismatch among the reward dataset, the dynamics dataset, and the test-time distribution. 
We present a generic approach to offline PLfO, called \underline{M}odality-agnoistic \underline{A}dversarial \underline{H}ypothesis \underline{A}daptation for \underline{L}earning from \underline{O}bservations (MAHALO). Built upon the concept of pessimism from offline RL literature~\cite{jin2021pessimism,liu2020provably,kumar2020conservative,xie2021bellman,cheng2022adversarially}, MAHALO optimizes for a performance lower bound accounting for insufficient data coverage on reward and dynamics. It can be realized by modifying existing offline RL algorithms based on adversarial training, such as~\cite{xie2021bellman,cheng2022adversarially,uehara2021pessimistic,rigter2022rambo,xie2022armor}. In particular, we present a model-free instantiation of MAHALO built upon ATAC~\cite{cheng2022adversarially}, an offline RL algorithm based on a Stackelberg game of relative pessimism.
In MAHALO, we consider the actor policy as the leader in the Stackelberg game, and adversarially train critic and reward functions so that they are data-consistent and can detect potential deficiency of the actor policy. As a result, the policy can be robust to the missing data coverage.

The contribution of this paper is two-fold. First, we propose offline PLfO, a novel formulation which relaxes data assumption for policy learning with offline data.
This general formulation encompasses most existing offline formulations, including, but not limited to, offline IL, ILfO, RL, and RL with unlabeled data. Second, we present MAHALO, a solution to offline PLfO based on pessimism. We further present a model-free realization of MAHALO. 
In theory and experiments, we show that MAHALO consistently outperforms or matches performance with more specialized algorithms across various offline PLfO scenarios and tasks.

\vspace{-1mm}
\section{Preliminaries}

\vspace{-1mm}
\paragraph{Markov Decision Process} We consider RL in a Markov Decision Process (MDP) $\MM=(\SS, \AA, P, R, \gamma)$, where $\SS$  and $\AA$ are the state and action spaces, $\gamma\in[0,1)$ is the discount factor, $P:\SS\times\AA\to \Delta(\SS)$ is the transition probability, where $\Delta(\cdot)$ denotes the space of probability distributions. We assume that the reward function $R$ is defined on state transitions
, i.e., $R: \SS\times\SS\to [0, R_{\max}]$, as we consider learning from observations. This state-transition reward function $R$ induces an effective state-action reward function $\Bar{R}(s,a) \coloneqq \E_{s'\sim P(\cdot|s,a)}[R(s,s')]$, which is the expected state-transition reward under the transition probability $P$. 
We denote a Markovian policy as $\pi:\SS\to\Delta(\AA)$.
The goal of RL is to find a policy which maximizes the expected discounted return $J(\pi)\coloneqq \E[\sum_{t=0}^\infty \gamma^t r_t]$, 
where $r_t=R(s_t, s_{t+1})$ and the expectation is over the randomness of running policy $\pi$ with transition probability $P$ starting from an initial state distribution $d_0(s)$.
For a policy $\pi$ and any function $f:\SS\times\AA\to \R$, we define the transition operator $\PP^\pi$ as $(\PP^\pi f)(s,a) \coloneqq \gamma\E_{s'\sim P(\cdot|s,a)}[f(s', \pi)]$, where $f(s',\pi)=\sum_{a'}\pi(a'|s')f(s',a')$.
For a policy $\pi$, we define the average state-action occupancy measure $d^\pi(s,a) \coloneqq (1-\gamma)\E[\sum_{t=0}^\infty \gamma^t \one(s_t=s, a_t=a)]$. 
We recall that $J(\pi) = \frac{1}{1-\gamma} \E_{s,a\sim d^\pi}\E_{s'\sim P(s'|s,a)}[R(s,s')]$. 

\vspace{-4mm}
\paragraph{Offline RL} 
Offline RL studies the problem of policy learning from a reward-labeled transition dataset $\DD = \{ (s,a,r,s') \}$. 
The goal of offline RL is to learn the best policy that can be explained by data, while not making assumptions on the data coverage quality. An offline RL algorithm ideally is able to learn the optimal policy of the MDP $\MM$, as long as the dataset covers the states and actions that the optimal policy would visit. 
Such robustness of offline RL to data coverage quality is commonly realized by pessimism, which reasons about the worst case for states and actions not covered by the offline data. Being pessimistic in the face of uncertainty naturally forces the agent to search for good policies within the data support. Typically, the pessimism is implemented via behavior regularization~\cite{fujimoto2021minimalist,wu2019behavior}, 
value penalty~\cite{jin2021pessimism,yu2020mopo,kumar2020conservative}, or adversarial training via a two-player game~\cite{cheng2022adversarially,xie2021bellman,rigter2022rambo,uehara2021pessimistic}. 
\vspace{-1mm}
\paragraph{Offline IL} 
Offline IL (such as behavior cloning) studies the problem of policy learning using only the transition dataset $\DD = \{ (s,a,s') \}$ without reward labels. The transition data is a union of near-optimal expert data $\DD_E$ and (optionally) a separately collected data of unknown quality $\DD_X$. 
Like offline RL, the data in offline IL does not have full coverage, and the principle of mimicking the expert data in IL also effectively encourages the learner to stay within the the data distribution. In fact, \citet{cheng2022adversarially} show that offline IL can be viewed as an offline RL problem with the largest reward uncertainty:
By running an offline RL algorithm that optimizes for the relative performance between the learner and the behavioral policy under data uncertainty, an IL mimicking behavior would naturally occur.



\vspace{-2mm}
\paragraph{Stackelberg game}
A Stackelberg game is a sequential two-player game~\cite{von2010market} between a leader $x$ and a follower $y$. In this game, the leader plays first and then the follower plays after seeing the leader's decision. The game can be written as a bilevel optimization problem:  $\max_x f(x,y_x) \textrm{ s.t. } y_x \in \max_y g(x,y)$, where $f$ and $g$ are the objectives of the leader and the follower, respectively.

\vspace{-1mm}
\section{Offline PLfO: A Unified Formulation for Offline RL and IL from Observations}
\vspace{-1mm}

In this section, we first introduce the generic setup of offline policy learning from observations (PLfO). 
Then we discuss practical scenarios where the data cannot be fully leveraged in offline RL and IL setups but is within the PLfO setup.

\vspace{-3mm}
\subsection{Problem Formulation}
\vspace{-1mm}

In offline PLfO, we assume access to pre-collected offline data consisting of transitions. In contrast to typical offline RL, 
we allow our data to include transitions which contain \emph{either} reward or action. In other words, in offline PLfO, we consider two datasets, a reward dataset $\DD_R = \{(s,r,s')\}$ and a dynamics dataset $\DD_\AA = \{(s,a,s')\}$. We note that these two datasets may not necessarily have an intersection.

We assume that both datasets are compliant with the underlying MDP. 
For the dynamics dataset, we follow standard compliance assumption in offline RL literature~\citep{jin2021pessimism}:
\begin{enumerate*}[label=\emph{\arabic*)}]
    \item for any $(s,a,s')\in \DD_R$, we have $s'\sim P(\cdot|s,a)$; and
    \item the state-action pairs $(s,a)$ in $\DD_\AA$ are sampled from the discounted state-action occupancy $d^\mu$ of a behavioral policy $\mu:\SS\to\Delta(\AA)$.
\end{enumerate*}
We slightly abuse notation to use $\mu$ to also denote the discounted state-action occupancy $d^\mu$, i.e., $\mu=d^\mu$.
For the reward dataset, for any $(s,r,s')\in\DD_R$, we assume that the reward function $R$ is defined on the state transition $(s,s')$ and $r=R(s,s')$. We do not make assumption on the underlying distribution of state transitions $(s,s')\sim\nu$, e.g., $s'$ in $(s,s')$ may not be sampled from $P(\cdot|s,a)$ for some action $a$. 
With a slight abuse of notation, we will also use $\nu$ to denote the underlying distribution of transition tuple containing reward $(s,r,s')$.
Like offline RL, we do not assume the coverage of these datasets on states and actions.
The goal of offline PLfO is to learn a policy $\pi$ which obtains high expected discounted return $J(\pi)$ in 
MDP $\MM$ while using reward and dynamics datasets of limited coverage.

\vspace{-1mm}
\subsection{Relation to Existing Formulations}\label{sec:exist-forms}
\vspace{-1mm}
Offline PLfO is a general formulation encompassing many existing problem setups. This means that an offline PLfO algorithm can solve any of the following problems, or the combination of them, via simple reductions.

\vspace{-4mm}
\paragraph{Offline RL}
Offline RL can be reduced to offline PLfO where the reward dataset and dynamics dataset are generated from the same underlying offline dataset $\DD = \{(s,a,r,s')\}$. 
The alignment assumption makes offline RL in general an easier problem than offline PLfO since there is no mismatch between reward and dynamics data.

\vspace{-4mm}
\paragraph{Offline RL with unlabeled data}\citet{yu2022leverage,singh2020cog,hu2023provable} consider a formulation where a subset of dynamics data is labeled with reward. 
In this scenario, the reward data is well-covered by dynamics data. The main challenge here is to leverage unlabeled dynamics data while not suffering from insufficient reward coverage. 

\vspace{-4mm}
\paragraph{Offline IL} Offline IL uses a dataset of expert demonstrations $\DD_E$.  
Although IL generally assumes no reward information, it makes an implicit assumption on expert performance. In other words, IL is a learning problem where only positive (i.e., high return) examples are given. 
This observation is also in line with existing work which takes a density matching perspective on IL~\cite{ho2016generative,kim2021demodice} and reward learning from demonstrations~\cite{fu2018learning,eysenbach2021replacing}.
\rev{Recent work such as~\cite{kim2021demodice,chang2021mitigating,smith2023strong} has considered offline IL with a separately-collected dynamics dataset $\DD_X$ of unknown quality. 
}

Offline IL can be reduced to offline PLfO: The reward dataset $\DD_R =\{(s, R_{\max}, s')\}$ comes from expert demonstrations, where $R_{\max}$ is the maximum reward. The dynamics dataset contains both the expert demonstrations \rev{$\DD_E$} and\rev{, if given, }the separately-collected dynamics data \rev{$\DD_X$}. 

\vspace{-1mm}
\paragraph{Offline ILfO} Offline ILfO is similar to offline IL, except that the expert demonstrations only contain state transitions, i.e., $\DD_E = \{(s_E, s_E')\}$. As such, offline ILfO can be viewed as offline PLfO with reward dataset $\DD_R = \{(s_E, R_{\max}, s_E')\}$ and dynamics dataset $\DD_\AA = \rev{\DD_X}$. Compared to the previous setups, offline ILfO faces an additional challenge of insufficient action coverage, as the expert state transitions $(s_E, s_E')$ may not be in the dynamics dataset $\DD_\AA$. 

\vspace{-1mm}
\subsection{Practical Scenarios}
\vspace{-1mm}

We now consider a few practical examples of reward and dynamics data sources. As we will see, 
it is likely that the coverage of reward data mismatches with that of dynamics data. In such scenario, offline PLfO formulation is well-suited as it can leverage \emph{all} available data. 

\vspace{-2mm}
\paragraph{Sources of dynamics data}
Since dynamics data is task-agnostic, dynamics data can be obtained from running data-collection policies on the MDP, as well as any other MDPs with the same dynamics. For example, in a multi-task setting~\cite{yu2020meta}, dynamics data can be acquired when solving different task rewards. However, it is expensive to label dynamics data with rewards for \emph{every} task. 
Realistically, data often only has reward information of the particular task that the data collection is for, which means that some state transitions only have actions, but not rewards. 

\vspace{-2mm}
\paragraph{Sources of reward data}
One practical strategy to reward labeling is to label randomly sampled trajectories from a pre-collected dynamics dataset~\cite{yu2022leverage}. This means that a large subset of the dynamics data remains unlabeled. Another setting is to re-use reward data collected by another agent (with the same state space). It is possible that the action information is unavailable or unusable in the underlying MDP. For example, the other agent can have different embodiment or use a different control modality. Additionally, reward information can be implicitly provided through expert demonstrations, as is discussed in \cref{sec:exist-forms}.

\vspace{-1mm}
\section{MAHALO}
\vspace{-1mm}

We propose a generic approach to offline PLfO, called Modality-agnostic Adversarial Hypothesis Adaptation for Learning from Observations (MAHALO). It is inspired by the idea of pessimism in offline RL~\cite{jin2021pessimism}. 

\vspace{-1mm}
\subsection{Solution Concept to Offline PLfO}
\vspace{-1mm}

In order to tackle the heterogeneous uncertainty in offline PLfO, we leverage the concept of version space in the offline RL literature~\cite{cheng2022adversarially,xie2021bellman,rigter2022rambo,uehara2021pessimistic,xie2022armor} to construct a performance lower bound for optimizing policies in MAHALO.
Here, a version space, denoted as $\JJ = \{ \hat{J} : \Pi \to \R \}$, is the space of policy performance hypotheses that remain feasible after observing data in $\DD_\AA$ and $\DD_R$. 
For example, if a bipedal robot has experienced that falling down receives zero rewards, then any hypothesis in the version space would give a zero reward to any policy that makes the robot fall, but the hypotheses in the version may disagree on which reward to give for other behaviors. 

In MAHALO we use the version space $\JJ$ to encapsulate uncertainty due to 
heterogeneous missing coverage. Because the version space 
$\JJ$ by definition includes the true performance function $J$, we can use $\JJ$ to construct a policy performance lower bound. 
For example, for a policy $\pi$ we can compute its absolute performance lower bound naturally as $\min_{\hat{J}\in \JJ } \hat{J}(\pi)$. Thus we can optimize policies through solving a saddle-point problem: 
$\max_{\pi\in\Pi} \min_{\hat{J}\in \JJ } \hat{J}(\pi)$
which provides a way to systematically optimize policy performance accounting for missing information in data.

For the case where the data are fully labeled with rewards and actions, offline RL literature has proposed several designs of the version space $\JJ$: with MDP models or value functions, in conjunction with absolute or relative pessimism~\cite{cheng2022adversarially,xie2022armor,xie2021bellman,rigter2022rambo,uehara2021pessimistic}. Here we generalize this technique to offline PLfO. In particular, we will design model-free versions of MAHALO, as model-free methods are simpler to implement, use less hyperparameters, and have demonstrated superior empirical performance in offline RL~\cite{yu2021combo}. 
\rev{In principle, MAHALO can be realized by \emph{any} version-space offline RL algorithm, including those based on models.}
%


\vspace{-1mm}
\subsection{Model-Free Realization of MAHALO}
\vspace{-1mm}

We now present a model-free realization of MAHALO based on the concept of relative pessimism in offline RL~\cite{cheng2022adversarially}. 
\rev{For clarity, we first present the formulation at the population level. We will later provide theoretical analysis for the finite-sample scenario in~\cref{sec:theory}. }
\rev{We introduce and analyze another realization of MAHALO based on absolute pessimism~\cite{xie2021bellman} in~\cref{sec:pspi}. }

We formulate offline PLfO as a Stackelberg game, with the actor policy $\pi\in\Pi$ as the leader. The followers consist of critic $f\in\FF$ and reward function $g\in\GG$.
\begin{align}\label{eq:mahalo}
        \hat{\pi} &\in \argmax_{\pi\in\Pi} \LL_\mu (\pi, f^\pi)\\
        \text{s.t.}\quad f^\pi &\in \argmin_{f\in\FF, g\in\GG} \LL_\mu (\pi, f) + \alpha \EE_\nu (g) + \beta \EE_\mu(\pi, f, g),\nonumber
\end{align}
with $\alpha\geq 0, \beta \geq 0$ being hyperparameters, and
\begin{align}
    \LL_\mu(\pi, f) &\coloneqq \E_\mu\big[f(s,\pi) - f(s, a)\big],\label{eq:pessimism}\\
    \EE_\nu (g) &\coloneqq \E_\nu \big[ \big(g(s,s') - r \big)^2 \big],\label{eq:reward-consistency}\\
    \EE_\mu(\pi, f, g) &\coloneqq \E_\mu\big[\big((f-\Bar{g}-\PP^\pi f)(s, a)\big)^2\big].\label{eq:bellman-consistency}
\end{align}
where $f(s,\pi)\coloneqq \E_{a\sim\pi(\cdot|s)}[f(s,a)]$ and $\Bar{g}(s,a) \coloneqq \E_{s'\sim P(\cdot|s,a)}[g(s,s')]$.
This optimization problem can be viewed as a regularized version of the constrained problem with a version space $\JJ =  \{ \hat{J} : \hat{J}(\pi) = 
J(\mu) + \frac{1}{1-\gamma}\LL_\mu(\pi, f),   \EE_\nu (g) \leq \epsilon_\alpha, \EE_\mu(\pi, f, g) \leq \epsilon_\beta \}$ for some $\epsilon_\alpha,\epsilon_\beta\geq0$ related to $\alpha,\beta$.\footnote{This analogy between the constrained and the regularized versions can be derived following the principle in \citep{xie2021bellman}.}  We adopt the regularized version, as it is easier to implement numerically.

To understand the above formulation, let us start by considering the last two terms in the followers' objective function. Recall that $\nu$ is the underlying distribution of reward data, so $\EE_\nu(g)$ measures whether the candidate reward function $g$ is data consistent. 
$\EE_\mu(\pi,f,g)$ quantifies whether the candidate critic $f$ is Bellman-consistent on the dynamics data for policy $\pi$ and reward function $g$. Therefore, with sufficiently large $\alpha$ and $\beta$, the reward $g^\pi$ and critic $f^\pi$ are both (approximately) consistent with the reward and dynamics data. 
On the other hand, $\LL_\mu(\pi, f)$ is the relative performance of between the candidate policy $\pi$ and behavioral policy $\mu$, with value estimated by critic $f$. The followers minimizes this quantity, meaning that $f^\pi$ provides a \emph{pessimistic} relative evaluation of policy $\pi$ (which will be formally shown in \cref{prop:rpi-main-txt}). The learned actor policy $\hat{\pi}$ maximizes this lower bound to improve over the the behavioral policy. 

\rev{We would like to stress on the importance of making the reward function additionally minimize for the Bellman error in the followers' objective~\eqref{eq:mahalo}. By minimizing~\eqref{eq:bellman-consistency}, the reward and critic functions work \emph{together} to provide a pessimistic evaluation of the policy. In other words, the Bellman error~\eqref{eq:bellman-consistency} connects the learned reward function to the pessimistic loss~\eqref{eq:pessimism}, since the reward function can change in a way to make the critic more pessimistic. 
}

\rev{The Stackelberg game for MAHALO in~\eqref{eq:mahalo} is similar to ATAC~\cite{cheng2022adversarially}. The main difference is that ATAC uses the observed reward in the Bellman error term as it assumes the reward is available for every transition. MAHALO, on the other hand, trains the reward function to be consistent with the reward data~\eqref{eq:reward-consistency} and Bellman-consistent with a pessimistic critic function~\eqref{eq:bellman-consistency} on the dynamics data.}

\rev{Below we discuss three desirable properties of MAHALO. First, given large dynamics and reward datasets, MAHALO can outperform any policy whose state-action distribution is well-covered by both datasets. Second, MAHALO ensures safe policy improvement~\cite{fujimoto2019off,laroche2019safe}, i.e.,  the learned policy $\hat{\pi}$ is no worse than the behavioral policy $\mu$ given sufficient data.} Third, MAHALO can automatically adapt to the structure within data. When applied to more restrictive formulations such as offline RL, offline IL, and offline ILfO, MAHALO shows similar behavior as specialized algorithms. 

\vspace{-2mm}
\subsubsection{Theoretical Properties}\label{sec:theory}
\vspace{-1mm}

%
\rev{
We analyze the solution to a finite-sample version of~\eqref{eq:mahalo} based on  dynamics dataset $\DD_\AA$ and reward dataset $\DD_R$:
\begin{small}
\begin{align}\label{eq:mahalo-finite}
        \hat{\pi} &\in \argmax_{\pi\in\Pi} \LL_{\DD_\AA} (\pi, f^\pi)\\
        \text{s.t.}\quad f^\pi &\in \argmin_{f\in\FF, g\in\GG} \LL_{\DD_\AA} (\pi, f) + \alpha \EE_{\DD_R} (g) + \beta \EE_{\DD_\AA}(\pi, f, g),\nonumber
\end{align}
\end{small}%
where $\LL_{\DD_\AA}(\pi, f)$ and $\EE_{\DD_R} (g)$ are the empirical estimates of $\LL_\mu(\pi, f)$ and $\EE_\nu(g)$, respectively, and
\begin{small}
\begin{align}
    \EE_{\DD_\AA}(\pi, f, g) &\coloneqq \E_{\DD_\AA}\big[\big(f(s,a) - g(s,s') - \gamma f(s',\pi)\big)^2\big]\label{eq:bellman-consistency-finite}\\
    &- \min_{f'\in\FF}\E_{\DD_\AA}\big[\big(f'(s,a) - g(s,s') - \gamma f(s',\pi)\big)^2\big].\nonumber
\end{align}
\end{small}%
The quantity $\EE_{\DD_\AA}(\pi, f, g)$ is the estimated Bellman error~\cite{antos2008learning}. We show in~\cref{sec:proofs} (\cref{lemma:concentration-mu}) that $\EE_{\DD_\AA}(\pi, f, g)$ can be used to approximate the Bellman error $\EE_{\mu}(\pi, f, g)$ defined in~\eqref{eq:bellman-consistency}.

For clarity purposes, we make a perfect realizability and completeness assumption below. Our analysis can be easily modified to consider an approximate version.
\begin{assumption}[Realizability \& Completeness]\label{assumption:realizability}
We assume  $\mu\in \Pi$, $R\in \GG$, and for all $\pi\in\Pi$, $Q^\pi\in \FF$. In addition, $\inf_{f'\in\FF} \| f' - \bar g- \PP^\pi f\|_\mu = 0  $, $\forall g\in\GG, f\in\FF$.
\end{assumption}

Due to the nature of offline learning, we will compare the performance of the learned policy $\hat{\pi}$ with a policy $\pi$ whose induced state-action occupancy $d^\pi$ is ``well-covered'' by data distributions.
Similar to~\cite{xie2021bellman,cheng2022adversarially,uehara2021pessimistic}, we define error transfer coefficients to measure distribution shift from the data distributions based on critic class $\FF$, and reward class $\GG$. 
This is a weaker notion than, e.g., density ratio 
~\citep{munos2008finite}.
\begin{definition}[Error transfer coefficients]
    The Bellman error transfer coefficient between $\rho\in\Delta(\SS\times\AA)$ and $\mu\in\Delta(\SS\times\AA)$ under policy $\pi$, critic class $\FF$ and reward  class $\GG$ is defined as
    \begin{equation}
        \mathcal{C}(\rho;\mu,\FF,\GG, \pi) \coloneqq \sup_{f\in\FF, g\in\GG} \frac{\|f - \bar{g} -\PP^\pi f\|_{2,\rho}^2}{\|f - \bar{g} -\PP^\pi f\|_{2,\mu}^2}.
    \end{equation}
    Similarly, the reward error transfer coefficient between $\rho$ 
    and $\nu\in\Delta(\SS\times\SS)$ under reward class $\GG$ is defined as
    \begin{equation}
        \mathcal{C} (\rho;\nu,\GG) \coloneqq \sup_{g\in\GG}\frac{\|\bar g - \bar R\|_{2,\rho}^2}{\|g - R\|_{2,\nu}^2}.
    \end{equation}
\end{definition}

We use $d_{\FF,\GG,\Pi}$ to denote the joint statistical complexity of critic class $\FF$, reward class $\GG$ and policy class $\Pi$, and use $d_\GG$ to denote the statistical complexity of reward class $\GG$ (e.g., when $\FF$, $\GG$ and $\Pi$ are all finite, we have
$d_{\FF,\GG,\Pi} = \OO(\log\sfrac{|\FF||\GG||\Pi|}{\delta})$ 
where $\delta$ is the failure probability). In~\cref{sec:proofs}, we establish statistical complexity using covering number.  We now state the main theoretical property of MAHALO of relative pessimism in~\eqref{eq:mahalo-finite}.

\begin{theorem}\label{thm:main-txt}
Under~\cref{assumption:realizability}, let $\hat{\pi}$ be the solution to~\eqref{eq:mahalo-finite} and let $\pi\in\Pi$ be any comparator policy. 
Let $C_1\geq 1, C_2\geq 1$ be constants, $\rho\in \Delta (\SS\times \AA)$ be a distribution that satisfy $\CC(\rho;\mu,\FF,\GG,{\pi})\leq C_1$ and $\CC(\rho;\nu,\GG)\leq C_2$. 
Define $\epsilon_\mu\coloneqq\sfrac{V_\text{max}^2 d_{\FF,\GG,\Pi}}{|\DD_\AA|}$,  $\epsilon_\nu\coloneqq\sfrac{R_\text{max}^2 d_{\GG}}{|\DD_R|}$ and
$\epsilon\coloneqq (\sqrt{C_1\epsilon_\nu} + \sqrt{C_2\epsilon_\mu})^2$. 
Choosing $\alpha = \Theta\left(\sfrac{V_\text{max}^{1/3} 
\epsilon^{1/3}}{\epsilon_\nu}\right)$ and $\beta=\Theta\left(\sfrac{V_\text{max}^{1/3} \epsilon^{1/3}}{\epsilon_\mu}\right)$, with high probability,
\begin{small}
\begin{align}\label{eq:mahalo-bound}
    &J(\pi) - J(\hat{\pi}) \\
    \leq &\OO\left(\frac{1}{1-\gamma}\left(\frac{C_1^{1/3} V_\text{max} (d_{\FF,\GG,\Pi})^{1/3}}{|\DD_\AA|^{1/3}} + \frac{C_2^{1/3} R_\text{max} (d_{\GG})^{1/3}}{|\DD_R|^{1/3}}\right)\right)\nonumber\\
    &+\underbrace{\frac{\langle d^\pi \setminus \rho , \bar g^\pi + \PP^\pi f^\pi - f^\pi\rangle}{1-\gamma}}_{\text{off-support error (dynamics)}} + \underbrace{\frac{\langle (d^\pi \ominus \mu)\setminus\rho, |\bar R - \bar g^\pi| \rangle}{1-\gamma}}_{\text{off-support error (reward)}},\nonumber
\end{align}
\end{small}%
where $(d^\pi \ominus \mu) \coloneqq d^\pi \setminus \mu + \mu \setminus d^\pi$ with $(d_1\setminus d_2)(s,a) \coloneqq \max (d_1(s,a) - d_2(s,a), 0)$ 
and $\langle \iota,f \rangle\coloneqq\sum_{(s,a)\in\SS\times\AA} \iota(s,a)f(s,a)$. 
\end{theorem}
The first term in~\eqref{eq:mahalo-bound} is the statistical error, which vanishes as $|\DD_\AA|,|\DD_R|\to\infty$. The second and third terms measure how much the comparator policy $\pi$ is outside of the support of dynamics and reward data distributions. In other words, our learned policy $\hat\pi$ can compete with any policy $\pi$ that is well-supported by both dynamics and reward data. 
%
Compared with Theorem 5 of ATAC in~\cite{cheng2022adversarially}, we have an extra term about the statistical error of the estimated reward, and an off-support reward error, because  we do not assume access to rewards on all transitions. 

Despite using partial reward labels, 
MAHALO has a robust policy improvement property, similar to  ATAC~\cite{cheng2022adversarially}, which guarantees that, for a known range of hyperparameters, the learned policy $\hat\pi$ is no worse than the behavioral policy by more than statistical errors. 
\begin{proposition}[Robust Policy Improvement]\label{prop:rpi-main-txt}
    Under \cref{assumption:realizability}, let $\hat\pi$ be the solution to~\eqref{eq:mahalo-finite}. Let $\epsilon_\mu$ and $\epsilon_\nu$ be as defined in~\cref{thm:main-txt}. 
    %
    We have, for any fixed $\alpha\geq 0$ and $\beta\geq 0$, with high probability,
    \begin{small}
    \begin{align}
        &J(\mu) - J(\hat\pi)\\ \leq & \OO\left(\frac{V_\text{max}}{1-\gamma}\sqrt{\frac{d_{\FF,\GG,\Pi}}{|\DD_\AA|}} + \frac{\alpha R_\text{max}^2 d_{\GG}}{(1-\gamma)|\DD_R|} + \frac{\beta V_\text{max}^2 d_{\FF,\GG,\Pi}}{(1-\gamma)|\DD_\AA|}\right).\nonumber
    \end{align}
    \end{small}
\end{proposition}
As $|\DD_\AA|$ and $|\DD_R|\to\infty$, we can see that the solution actor policy $\hat{\pi}$ is guaranteed to be no worse than the behavioral policy $\mu$ with \emph{any} choices of fixed $\alpha,\beta \geq 0 $. 
}

\input{algo}

\subsubsection{Adaption to Structure within Data}
\vspace{-1mm}

Since MAHALO can be used to solve offline PLfO, MAHALO can be used to solve more restrictive problems via simple reductions. Then, a natural question is: Can MAHALO achieve similar performance as specialized algorithms? The answer is yes. Below we show that this is because MAHALO can adapt to the hidden structure within data despite being agnostic to \rev{the relationship between dynamics and reward datasets}.

\vspace{-3mm}
\paragraph{Offline RL} In offline RL, since the reward and dynamics datasets are aligned, 
we have
$\EE_\nu(g) = \E_\mu[\E_{s'\sim P(\cdot|s,a)}[(g(s,s') - R(s,s'))^2]]$. 
The expected Bellman error of critic $f$ (with the true reward $R$), $\EE_\mu (\pi, f, R)$, can be upper bounded by $2\EE_\nu(g) + 2\EE_\mu(\pi, f, g)$.
%
With sufficiently large $\alpha$ and $\beta$, 
for any actor policy $\pi$, its corresponding critic $f^\pi$ is 
an approximately Bellman-consistent critic function. Therefore, MAHALO behaves similarly as ATAC~\cite{cheng2022adversarially}. 
\rev{This can also be seen from~\cref{thm:main-txt}. Since $|\DD_\AA|=|\DD_R|$, the statistical error is dominated by the first term, which is the same as the statistical error as ATAC. 
In offline RL, good dynamics coverage implies good data coverage, i.e., we have $(d^\pi \ominus \mu) \setminus \rho \leq d^\pi \setminus \rho$. This gives us a similar off-support error term.}

\vspace{-3mm}
\paragraph{Offline RL with unlabeled data} 
For the sake of simplicity, we consider a tabular setting. In this case, regardless of the policy $\pi\in\Pi$, with sufficiently large $\alpha$, the reward function $g^\pi$ such that $g^\pi(s,s')\approx R(s,s')$ when $(s,s')$ is within coverage of reward data $\nu$; \rev{the value of $g^\pi(s,s')$ on other states would adapt pessimistically according to the learner policy.}
The critic $f^\pi$ is effectively conducting a pessimistic policy evaluation for $\pi$ in such a reward function. This means that MAHALO in the tabular setting has a similar behavior as UDS~\cite{yu2022leverage}, a strategy where zero reward is given to \emph{all} unlabeled data. The difference, though, is that MAHALO still assigns an accurate reward to unlabeled dynamics transitions $(s,a,s')$ when $(s,s')$ is within coverage of $\nu$. This implies that MAHALO induces less bias than UDS, even though UDS knows more about the underlying data generation process. 

\vspace{-2mm}
\paragraph{Offline IL and ILfO} 
\rev{In the simplest offline IL setting where only the set of expert demonstrations $\DD_E$ is given, \cref{prop:rpi-main-txt} (with  $\pi_E=\mu$)  shows that the learned policy $\hat\pi$ is no worse than the expert policy $\pi_E$ (in terms of $R(s,s')=R_{\max} \one[(s,s')\in \text{supp}(\nu)]$) up to statistical errors. 
Now consider the scenario where we have access the expert demonstrations $\DD_E$ (which may or may not have actions)\footnote{\rev{When $\DD_E$ contains action, we can alternatively replace $\LL_{\DD_A}(\pi, f)$ with $\LL_{\DD_E}(\pi, f)$, which would ensure robust policy improvement to the expert policy.}} and a separately collected dynamics dataset $\DD_X$ with unknown quality.
\cref{thm:main-txt} (with $R(s,s')=R_{\max} \one[(s,s')\in \text{supp}(\nu)]$) shows that the learned policy would stay within the support of the expert distribution, similar to \citep{wang2019random,smith2023strong}.
}
%
%

\vspace{-5mm}
\paragraph{Summary}
MAHALO is an general, data-agnostic algorithm for solving offline PLfO problems. 
When applied to more restrictive settings when data presents additional structure, MAHALO can automatically adapt to such structure, and achieves behavior on par with existing specialized algorithms. This means that MAHALO can leverage broad sources of data and no special care, e.g. data alignment or management, needs to be taken during data collection. 

\input{tables/new_d4rl_results.tex}

\input{tables/scenarios.tex}

\input{tables/new_mw_results.tex}

\vspace{-1mm}
\subsection{Implementation}
\vspace{-1mm}

The MAHALO realization above can be implemented by making a few simple modifications to ATAC~\cite{cheng2022adversarially}. 
The resulting algorithm is presented in \cref{alg:mahalo-atac}, with the modifications marked in \highlight{magenta}. 
This implementation of MAHALO is based on a reduction of two-player game to no-regret policy optimization~\cite{cheng2022adversarially}. 
We use ADAM~\cite{kingma2014adam} optimizer with a faster learning rate $\eta_\text{fast}$ for the critic and reward functions, and a smaller learning rate $\eta_\text{slow}$ for the actor.

\vspace{-2mm}
\subsubsection{Actor and Critic Update}
\vspace{-1mm}
We use a strategy for updating the critic similar to ATAC. 
ATAC uses a surrogate to the Bellman error term in~\cref{eq:bellman-consistency} called double Q residual algorithm (DQRA) loss, which combines double Q heuristics~\cite{fujimoto2018addressing}, the residual algorithm~\cite{baird1995residual}, and target networks~\cite{mnih2015human}. 
The critic is parameterized by two networks $\{f_1, f_2\}$, each with a delayed target $\{\Bar{f}_1, \Bar{f}_2\}$. The target value is computed by taking the minimum of the two targets $\Bar{f}_{\min}(s,a) = \min_{i\in\{1,2\}} \Bar{f}_i(s,a)$.
DQRA uses a convex combination of the temporal difference (TD) losses of the critic network and target networks to stabilize learning. For $i\in\{1, 2\}$, the DQRA loss is defined as
\begin{align}
    \EE_{\DD_\AA^\text{mini}}^w(\pi, f_i, g) &\coloneqq (1-w)\EE_{\DD_\AA^\text{mini}}^\text{td}(\pi, f_i, f_i, g)\\
    &\qquad + w \EE_{\DD_\AA^\text{mini}}^\text{td}(\pi, f_i, \Bar{f}_{\min},g),\nonumber
\end{align}
where $w\in[0,1]$ is the weight and the TD loss is given by
\begin{equation*}\small
    \EE_{\DD_\AA^\text{mini}}^\text{td}(\pi, f, f',g) \coloneqq \E_{\DD_\AA^\text{mini}}[(f(s,a) - g(s,s') - \gamma f'(s', \pi))^2].
\end{equation*}
Note that here we use predicted reward $g(s,s')$ since reward $r$ is not observed in $\DD_\AA$. After each gradient update, we apply an $\ell_2$ projection on critic network weights (not on bias terms) (line 7) as is done in ATAC~\cite{cheng2022adversarially}.
We update the actor using the same way as ATAC. 
The actor policy optimizes for a \emph{single} critic $f_1$ 
and uses a Lagrangian relaxation of minimum entropy constraint similar to SAC~\cite{haarnoja2018soft}. 

\vspace{-1mm}
\subsubsection{Reward Update}
\vspace{-1mm}
The major difference between \cref{alg:mahalo-atac} and ATAC is the reward function update. 
We estimate the reward prediction loss empirically from minibatches sampled from the reward dataset $\DD_R$: $\EE_{\DD_R^\text{mini}} (g) = \E_{\DD_R^\text{mini}}[(g(s,s') - r)^2]$. The reward loss is the weighted sum of reward prediction loss $\EE_{\DD_R^\text{mini}} (g)$ and DQRA losses $\sum_{i=\{1,2\}}\EE_{D^\text{mini}_\AA}^w (\pi, f_i, {g})$ (line 6).  The DQRA losses connect the reward to the critic which minimizes also the performance difference; as a result, the learned reward function is also pessimistically estimated. In other words, the critic and the reward functions jointly form a hypothesis that adversarially adapts to the learner's policy.

\vspace{-2mm}
\section{Experiments}
\vspace{-1mm}
We aim to answer the following questions: 
\begin{enumerate*}[label=(\alph*)]
    \item Is MAHALO effective in solving different instances of offline PLfO problems?
    \item Can MAHALO achieve similar performance as other specialized algorithms?
    \item Whether MAHALO can obtain comparable performance to oracle algorithms with full reward and dynamics information?
    \item In what situation is the pessimistic reward function of MAHALO critical in achieving good performance?
\end{enumerate*}

\vspace{-4mm}
\paragraph{Scenarios} To answer question (a), we design five instances of offline PLfO inspired by practical scenarios. \textbf{ILfO:} The learner is presented with a relatively small set of expert-level state-only trajectories, and a dynamics dataset of mixed quality. The dynamics data contains trajectories collected by policies with different performance-level. \textbf{IL:} Similar to ILfO, but we additionally provide expert actions to the learner. \textbf{RLfO:} Similar to ILfO, but the learner is additionally given the reward along the expert trajectories. This simulates the scenario where we provide manual labels to the expert demonstrations. \textbf{RL-expert:} The learner is presented with the expert action in addition to what is given in RLfO. \textbf{RL-sample:} The learner is provided with the mixed quality dynamics data, with a subset of trajectories labeled with reward. The information available in each scenario is summarized in~\cref{tab:scenarios}.

\vspace{-0.1in}
\paragraph{Baselines} To address question (b), we consider a few specialized baseline algorithms for each scenario. \rev{We consider behavior cloning from observation (\textbf{BCO})~\cite{torabi2018behavioral} as a baseline for ILfO, and behavior cloning (\textbf{BC}) for IL. We also include \textbf{SMODICE}~\cite{ma2022versatile}, a state-of-the art offline ILfO algorithm as a baseline algorithm for ILfO, and a variation of SMODICE which uses a state-action discriminator as a baseline for offline IL.}
Since IL, RL-expert and RL-sample can be viewed as RL with unlabeled data, we present two baselines for these settings: running an offline RL algorithm, we use \textbf{ATAC}~\cite{cheng2022adversarially} since it is the closest to MAHALO, only on labeled data and \textbf{UDS}~\cite{yu2022leverage}. We note that UDS 
requires knowing the common transitions between reward and dynamics datasets. 
We implement UDS 
with ATAC, which is slightly different than~\cite{yu2022leverage},where CQL~\cite{kumar2020conservative} is used. Since learning inverse dynamics is a common approach to learning with observations~\cite{torabi2018behavioral,edwards2020estimating}, we implement a baseline called action prediction (\textbf{AP}). It pretrains an inverse dynamics model on the dynamics dataset, predicts the missing actions in the reward dataset, and runs ATAC~\cite{cheng2022adversarially} on the reward dataset. 
For question \rev{(d)}, we implement a baseline algorithm called reward prediction (\textbf{RP}). It pretrains a reward function on the reward dataset. This reward function is then fixed for \rev{offline RL training using ATAC~\cite{cheng2022adversarially}}. Comparing MAHALO with RP can give us information on whether the adversarial training of reward function in MAHALO is effective.
Finally, \rev{to answer question (c)}, we train ATAC on a fully-labeled dynamics dataset (\textbf{Oracle}) to evaluate if MAHALO can achieve comparable performance with a privileged offline RL algorithm (which has access to more information). 

We evaluate MAHALO and above-mentioned algorithms on two sets of environments: locomotion tasks from the D4RL benchmark~\cite{fu2020d4rl} and robot manipulation tasks from the Meta-World domain~\cite{yu2020meta}. 

\vspace{-2mm}
\subsection{Evaluation on D4RL}
\vspace{-1mm}

We consider three environments from D4RL~\cite{fu2020d4rl}: hopper-v2, walker2d-v2, and halfcheetah-v2. 
The rewards in the three environments promote moving forward. In hopper and walker2d, the agent is required to stay within a health height range, otherwise the episode terminates. 
We construct a mixed quality dynamics dataset with 3.4 M transitions through concatenating the random, medium, medium-replay, and full-replay datasets. The expert data is consisted of 10k transitions ($\sim10$ trajectories) generated by randomly sampling trajectories from the expert dataset. For RL-sample, we sample 34k transitions (1\% of overall data).

The normalized return for five scenarios for each dataset is listed in~\cref{tab:mujoco}. MAHALO achieves top performance in almost every task except halfcheetah. 
MAHALO also shows comparable performance with the privileged oracle.
Reward prediction (RP) also performs well in all RL tasks. 
It, however, does not perform as well in IL tasks, since the reward function of RP can not generalize beyond the constant training reward. UDS 
does not perform well in these scenarios potentially due to being overly pessimistic. \rev{SMODICE uniformly performs worse than MAHALO in hopper and walker tasks. We hypothesize that the fixed discriminator reward function of SMODICE becomes overly pessimistic and discourages policy learning when the expert and dynamics distribution mismatches with one another.}

\vspace{-2mm}
\subsection{Evaluation on Meta-World}
\vspace{-1mm}

We additionally evaluate MAHALO and other baseline algorithms on five robot manipulation tasks from Meta-World~\cite{yu2020meta}. 
For Meta-World tasks, we use a non-positive reward function that promotes agent to reach the goal, which is a terminal state, as fast as possible. 
%
We generate a mixed quality dataset by adding different level of noises to a scripted policy provided by Meta-World for each task. We collect $100$ trajectories each for zero-mean Gaussian noise with standard deviations $[0.1, 0.5, 1.0]$, and use these $300$ trajectories as the mixed quality dataset. We separately collect $100$ expert trajectories. For, RL-sample, we randomly sample trajectories to label 50\% of the dataset. Note that we use more reward data for Meta-World since the state space (which includes the space of goals) is larger.

We observe that MAHALO is one of the overall best-performing algorithms. ATAC also achieves strong results in IL and RL-expert. This is because ATAC, in these scenarios, is only presented with expert data (we do not see similar effect in D4RL tasks since the expert dataset is much smaller there). UDS 
shows similar performance as MAHALO in RL scenarios. 
They, however, perform worse than MAHALO in IL. Reward prediction (RP) performs worse than MAHALO in most cases, especially in IL and ILfO scenarios. 
This shows that the pessimistic reward function in MAHALO is critical in achieving robust performance across different tasks and scenarios. 
\rev{
SMODICE is not able to train reliably, and often diverges during training. We therefore take the success rate of the best performing policy during training instead of the final one. We find that the best policy of SMODICE  underperforms the final policy of most algorithms, including MAHALO, potentially due to this instability in training.
}

\vspace{-3mm}
\section*{Acknowledgements}
We thank Andrey Kolobov for sharing data collected on the Meta-World tasks. We thank Mohak Bhardwaj for providing the script for processing results and Sinong Geng for helpful discussions on Meta-World experiments.
\bibliography{ref}
\bibliographystyle{icml2023}

\newpage
\appendix
\onecolumn
\section{Related Work}\label{sec:related-work}

\paragraph{Offline RL}
Existing work on offline RL can be broadly classified into model-based approaches 
\cite{yu2020mopo, kidambi2020morel, yu2021combo,uehara2021pessimistic,rigter2022rambo} and model-free approaches \cite{jin2021pessimism,fujimoto2021minimalist,xie2021bellman,cheng2022adversarially, kumar2020conservative, kostrikov2021offline}. The main challenge to offline RL is the mismatch between the offline dataset and test-time distribution. Two common strategies to offline RL are behavior regularization
\cite{fujimoto2021minimalist, wu2019behavior,kostrikov2021offline} (restricting policy to be close to the behavioral policy) and pessimism~\cite{jin2021pessimism,liu2020provably,xie2021bellman,cheng2022adversarially, kumar2020conservative}. Our paper follows more closely with model-free approaches built upon the concept of pessimism. In particular, we draw inspirations from approaches which conduct pessimistic policy evaluation on a version space of data-consistent hypotheses~\cite{xie2021bellman,cheng2022adversarially,uehara2021pessimistic,rigter2022rambo}.
However, we consider a more general formulation than offline RL, where reward or action can be missing from a subset of data. \citet{edwards2020estimating} approach RL from observation by combining offline RL with a separately learned inverse dynamics model. It however makes implicit assumption that the dynamics data is abundant. 

\paragraph{Offline RL with unlabeled data} \citet{yu2022leverage} and \citet{singh2020cog} consider a setting that reward can be missing from a subset of data, and uses a strategy of giving zero reward to these transitions. We show that MAHALO has a similar behavior when applied to this setting, while incurring less bias by correctly labeling rewards to transitions which are in support of the reward data. \rev{Recently, a strategy called PDS~\cite{hu2023provable} is proposed to construct a pessimistic reward function using a ensemble of neural networks. MAHALO differs from PDS as the pessimistic reward function is constructed together with a pessimistic critic function. We also provide theoretical analysis of MAHALO in a general function approximator setting while PDS only has theoretical guarantees for linear MDPs.  \citet{li2023mind} consider a related setting where rewards can be mislabeled or imperfect. This is slightly different to our problem formation, where the reward (and action) for each transition is either correct or missing.}

\paragraph{Offline IL} Our work is also related to offline IL, where the expert dataset and the optional, often separately-collected, dynamics dataset can have different coverage. \citet{zolna2020offline} adapt GAIL~\cite{ho2016generative} to an offline setting. \citet{kim2021demodice} seek to match the state-action occupancy using distribution correction estimation. \citet{chang2021mitigating} minimize the state-action occupancy divergence using a pessimistic model. \rev{More recently, DWBC~\cite{xu2022discriminator} proposes to approach offline IL through weighted behavior cloning, where the weights are given by a discriminator. CLARE~\cite{yue2023clare} extends maximum entropy inverse reinforcement learning to an offline setting with an unknown quality dynamics dataset.} \citet{smith2023strong} propose to run offline RL with a binary reward indicating whether the transition is generated by the expert. These approaches require knowing the actions in expert demonstrations, which is a more restrictive assumption than our formulation.

\paragraph{IL from observations} ILfO further relaxes the assumption of observing expert actions. Earlier approaches~\cite{torabi2018behavioral,torabi2019adversarial,torabi2019recent,yang2019imitation,schmeckpeper2021reinforcement} focus more on the online setting, where additional data collection is allowed when training the policy. There are a few recent papers that consider the offline ILfO setting. For example, \citet{zhu2020off,ma2022versatile} take an off-policy distribution matching approach, and \citet{kidambi2021mobile} use a model-based approach to minimax IL. These work, however, assumes a near-optimal expert. MAHALO, however, can work with both near-optimal expert trajectories and non-expert state-only trajectories labeled with reward.

\paragraph{Learning with dynamics mismatch}
Learning from observation is often closely related to learning with dynamics mismatch
\cite{liu2019state,desai2020imitation,cao2021learning,radosavovic2021state,cao2021feasibility,ma2022versatile}, as it is hard to directly use action information from a different MDP. 
Although MAHALO can potentially handle this setting, we consider it beyond the scope of this paper and defer it to future work. 

\paragraph{Learning reward functions} When applied to offline IL/ILfO settings, our approach also has connection with work on reward learning from demonstrations~\cite{ng2000algorithms,abbeel2004apprenticeship,ziebart2008maximum,fu2018learning,fu2018variational,singh2019end,eysenbach2021replacing,konyushkova2020semi,kim2020domain}, as MAHALO produces a reward function. The main difference between MAHALO and this line of work is that we do not treat reward learning and policy learning as two separate processes. We view the learned reward function rather as a by-product of the algorithm.

\section{Implementation Details}\label{sec:exp-details}

\subsection{Neural Network Architectures}
For MAHALO and all baseline algorithms, we parameterize the policy and (if applicable) critic networks with fully connected neural networks with $3$ hidden layers of size $256$. The policy is implemented as a tanh Gaussian distribution, where the mean and standard deviations are predicted by the two heads of the policy network. For MAHALO, we use the same network structure for the reward function. Our code is adapted from \href{https://github.com/chinganc/lightATAC}{https://github.com/chinganc/lightATAC}.

\subsection{Hyperparameters}
Since both MAHALO and most baseline algorithms are implemented via modifying ATAC~\cite{cheng2022adversarially}, we use similar choices of hyperparameters across different algorithms. 
We use a fixed learning rate across all algorithms and experiments. Same as~\citet{cheng2022adversarially}, we use $\eta_\text{slow}=5\times 10^{-7}$ for policy updates and $\eta_\text{fast}=5\times 10^{-4}$ for updating the critic (and reward for MAHALO). We use a common discount factor of $\gamma=0.99$ for all experiments. For target update, we use $\tau=0.005$, same as~\citet{cheng2022adversarially,haarnoja2018soft}. We use a fixed batch size of $256$. Our ATAC implementation is from \href{https://github.com/chinganc/lightATAC}{https://github.com/chinganc/lightATAC}.

To tune hyperparameter $\beta$, we run all algorithms with $\beta\in[0.1, 1.0, 10.0, 100.0, 1000.0]$ and report results using the best $\beta$ for each scenario and task. For MAHALO, we tune hyperparameter $\alpha$ by experimenting with $(\alpha/\beta)\in[10.0, 100.0, 1000.0, 10000.0, 100000.0]$. To make it a fair comparison with other baseline algorithms, we report results generated by a \emph{fixed} ratio $(\alpha/\beta)$ for each set of results. We use a fixed ratio of $(\alpha/\beta)\equiv100000.0$ for D4RL tasks and $(\alpha/\beta)\equiv100.0$ for Meta-World. 

\rev{For SMODICE~\cite{ma2022versatile}, we adapt the implementation from \href{https://github.com/JasonMa2016/SMODICE}{https://github.com/JasonMa2016/SMODICE} and use the default parameters therein. We additionally experiment with SMODICE using $\chi^2$ divergence (SMODICE uses KL divergence by default) and report its performance in~\cref{tab:mujoco-std} and~\cref{tab:mw-rewards-std} in~\cref{sec:std}.}

\subsection{Training}

For all algorithms based on ATAC, we warm-start training with 1) behavior cloning the behavioral policy and 2) learning a critic function to match the value of the behavioral policy. For MAHALO, we additionally train the reward function in this phase. For D4RL, we ran 1k steps of pretraining for D4RL tasks and 1M steps of pretraining for Meta-World. We then run ATAC and MAHALO updates for 1M steps and report the results.


\rev{
\section{Theoretical Analysis}\label{sec:proofs}

\subsection{Concentration Analysis}
In this section, we state a few concentration results useful in finite-sample regime. These results can be obtained by straightforward modifications of the analysis in~\cite{cheng2022adversarially}. As such, we will not provide detailed proofs. We will instead point out the the correspondence of each lemma in~\cite{cheng2022adversarially} and required modifications.
We first provide the definition of covering number, which will be used to establish concentration results.
\begin{definition}[$\epsilon$-covering number]\label{def:covering-number}
An $\epsilon$-cover of a set $\Phi$ with respect to a metric $d(\cdot,\cdot)$ is a set $\{\tilde \phi_1,\ldots,\tilde \phi_n\}\subseteq\Phi$ such that for each $\phi\in\Phi$, there exists some $\tilde \phi_i\in\{\tilde \phi_1,\ldots,\tilde \phi_n\}$ such that $d(\phi, \tilde \phi_i)\leq \epsilon$. We define the $\epsilon$-covering number of a set $\Phi$ under a metric $d$ to be the cardinality of the smallest $\epsilon$-cover, denoted $\NN_d(\Phi,\epsilon)$. 

In particular, we use $\NN_\infty(\FF,\epsilon)$ to denote the $\epsilon$-covering number under $\ell_\infty$ norm on set $\FF \subseteq (\SS\times\AA\to[0,V_\text{max}])$:
\begin{equation}
    d_\FF(f_1, f_2) \coloneqq \|f_1 - f_2\|_\infty = \sup_{(s,a)\in\SS\times\AA}|f_1(s,a)-f_2(s,a)|.
\end{equation}
Similarly, we use $\NN_\infty(\GG,\epsilon)$ to denote the $\epsilon$-covering number under $\ell_\infty$ norm on set $\GG \subseteq (\SS\times\SS \to [0, R_\text{max}])$. Finally, we define metric for the policy class as
\begin{equation}
    d_\Pi(\pi_1, \pi_2) \coloneqq \|\pi_1 - \pi_2\|_{\infty,1} = \sup_{s\in\SS}\|\pi_1(\cdot|s) - \pi_2(\cdot|s)\|_1,
\end{equation}
and denote the corresponding $\epsilon$-covering number as $\NN_{\infty,1}(\Pi, \epsilon)$.
\end{definition}

We first establish the concentration results for $\EE_{\DD_\AA}(\pi, f, g)$. The following lemma can obtained by modifying Theorem 9 from~\cite{cheng2022adversarially} by additionally considering an $\frac{R_\text{max}}{|\DD_\AA|}$-cover of $\GG$. 
\begin{lemma}\label{lemma:concentration-mu}
Suppose \cref{assumption:realizability} holds. With probability at least $1-\delta$, for any $\pi\in\Pi$ and $f\in\FF$, $g\in\GG$,
\begin{align}
    \sqrt{\EE_\mu(\pi, f, g)} - \sqrt{\EE_{\DD_\AA}(\pi, f, g)}  &\leq \OO \left(V_\text{max} \sqrt{\frac{ \log ( |\NN_\infty(\FF, V_\text{max} / |\DD_\AA|)| |\NN_\infty(\GG, R_\text{max} / |\DD_\AA|)| |\NN_{\infty,1}(\Pi, 1 / |\DD_\AA|)| / \delta)}{|\DD_\AA|}}\right) \nonumber\\
    &\eqqcolon \OO(\sqrt{\epsilon_\mu}),
\end{align}
\end{lemma}

We then consider $\EE_{\DD_\AA}(\pi, Q^\pi, R)$. The following Lemma is a direct result of Theorem 8 from~\cite{cheng2022adversarially} when~\cref{assumption:realizability} holds. 
\begin{lemma}\label{lemma:concentration-mu-q}
Suppose $\Pi$, $\FF$ and $\GG$ satisfies~\cref{assumption:realizability}. With probability at least $1-\delta$, for any $\pi\in\Pi$,
\begin{equation}
    \EE_{\DD_\AA}(\pi, Q^\pi, R) \leq \OO \left(\frac{V_\text{max}^2 \log ( |\NN_\infty(\FF, V_\text{max} / |\DD_\AA|)| 
    |\NN_{\infty,1}(\Pi, 1 / |\DD_\AA|)| / \delta)}{|\DD_\AA|}\right) \leq \OO(\epsilon_\mu).
\end{equation}
\end{lemma}

We can also modify Theorem 8 and 9 from~\cite{cheng2022adversarially} to provide concentration results on $\EE_{\DD_R}(g)$ and $\EE_{\DD_R}(R)$. This can be done by considering reward class $\GG$ instead $\FF$ (which also means that we need $\frac{R_\text{max}}{|\DD_R|}$-cover of $\GG$).
\begin{lemma}\label{lemma:concentration-nu}
With probability at least $1-\delta$, for any $g\in\GG$,
\begin{equation}
    \sqrt{\EE_\nu(g)} - \sqrt{\EE_{\DD_R}(g)} = \OO\left(R_\text{max}\sqrt{\frac{ \log ( \NN_\infty(\GG, R_\text{max} / |\DD_R|) / \delta)}{|\DD_R|}}\right) \eqqcolon \OO(\sqrt{\epsilon_\nu}),
\end{equation}
\end{lemma}

\begin{lemma}\label{lemma:concentration-nu-r}
With probability at least $1-\delta$, 
\begin{equation}
    \EE_{\DD_R}(R) \leq \OO \left(\frac{R_\text{max}^2 \log ( \NN_\infty(\GG, R_\text{max} / |\DD_R|) / \delta)}{|\DD_R|}\right) = \OO(\epsilon_\nu).
\end{equation}
\end{lemma}

Finally, the concentration result for $\LL_{\DD_\AA}(\pi, f)$ can be obtained by Hoeffding's inequality.
\begin{lemma}\label{lemma:concentration-l}
With probability at least $1-\delta$, for any $\pi\in\Pi$ and $f\in\FF$,
\begin{equation}
    \big|\LL_\mu(\pi, f) - \LL_{\DD_\AA}(\pi, f)\big| \leq \OO\left(V_\text{max} \sqrt{\frac{ \log ( |\NN_\infty(\FF, V_\text{max} / |\DD_\AA|)| |\NN_{\infty,1}(\Pi, 1 / |\DD_\AA|)| / \delta)}{|\DD_\AA|}}\right) \leq \OO(\sqrt{\epsilon_\mu}).
\end{equation}
\end{lemma}

\subsection{Auxiliary Lemmas}
In this section, we provide two auxiliary lemmas for showing the main results. These lemmas are stated in a general manner. For example, we use $\mu$ to denote a probability distribution on $\SS\times\AA$ rather than the dynamics data distribution. Similarly, we use $f$ to denote an arbitrary function $f:\SS\times\AA\to\R$.

The first lemma bounds the expectation of a function $f$ on an arbitrary distribution $\mu$ by the execution of its absolute value on another distribution $\rho$ and an inner product of difference of the two distributions and $f$. In later analysis, we will use this lemma to decompose the expectation on an arbitrary distribution into an ``in-support'' expectation and an ``off-support'' term.
\begin{lemma}\label{lemma:distr-shift}
    For any $\mu, \rho \in \Delta(\SS\times \AA)$, and any $f:\SS\times\AA\to \R$,
    \begin{equation}
            \E_\mu[f(s,a)] \leq 2\E_\rho[|f(s,a)|] + \langle \mu \setminus \rho , f\rangle
    \end{equation}
    where $(\mu\setminus\rho)(s,a) \coloneqq \max (\mu(s,a) - \rho(s,a), 0)$ and $\langle \iota,f \rangle\coloneqq\sum_{(s,a)\in\SS\times\AA} \iota(s,a)\cdot f(s,a)$ for any $\iota:\SS\times\AA\to\R$.
\end{lemma}
\begin{proof}
    This lemma can be shown by decomposing probability measure $\mu$ into probability measure $\rho$ and a signed measure $\mu - \rho$, and then splitting the signed measure $\mu-\rho$ into its positive and negative parts.
    \begin{align*}
            \E_\mu[f(s,a)] &= \sum_{(s,a)} \mu(s,a) \cdot f(s,a)\\
            &= \sum_{(s,a)} \rho(s,a) \cdot f(s,a) + \sum_{(s,a)} (\mu(s,a) - \rho(s,a)) \cdot f(s,a)\\
            &=\E_\rho[f(s,a)] + \sum_{(s,a)} (\mu(s,a) - \rho(s,a)) \cdot  f(s,a)\\
            &=\E_\rho[f(s,a)] + \sum_{(s,a)} \one(\mu(s,a) > \rho(s,a)) (\mu(s,a) - \rho(s,a)) \cdot f(s,a)\\
            &\qquad\qquad\qquad + \sum_{(s,a)} \one(\mu(s,a) \leq \rho(s,a)) (
    \mu(s,a)- \rho(s,a))\cdot f(s,a)\\
            &=\E_\rho[f(s,a)] + \langle \mu \setminus \rho , f\rangle + \sum_{(s,a)} \one(\rho(s,a) \geq \mu(s,a)) (\rho(s,a) - \mu(s,a))\cdot (-f(s,a))\\
        &\leq \E_\rho[|f(s,a)|] + \langle \mu \setminus \rho , f\rangle + \sum_{(s,a)} \one(\rho(s,a) \geq \mu(s,a)) (\rho(s,a) - \mu(s,a))\cdot |f(s,a)|\\
        &\leq \E_\rho[|f(s,a)|] + \langle \mu \setminus \rho , f\rangle + \sum_{(s,a)} \rho(s,a)\cdot |f(s,a)|\\
        &= 2\E_\rho[|f(s,a)|] + \langle \mu \setminus \rho , f\rangle
    \end{align*}
    where the first inequality follows from $|f(s,a)|\geq f(s,a)$ and $|f(s,a)|\geq -f(s,a)$, and the second inequality follows from  $\rho(s,a)\geq \one(\rho(s,a) \geq \mu(s,a)) (\rho(s,a) - \mu(s,a)) \geq 0$ and $|f(s,a)|\geq 0$.
\end{proof}

The second lemma bounds the $\ell_1$ norm of a function $f$ under a probability measure $\mu$ by its $\ell_2$ norm under the same measure. 
\begin{lemma}\label{lemma:l1-to-l2}
    For any $\mu\in\Delta(\SS\times\AA)$, and $f:\SS\times\AA \to \R$,
    \begin{equation}
        \E_\mu \big[|f(s,a)|\big] \leq \sqrt{\E_\mu\big[\big(f(s,a)\big)^2\big]} = \|f\|_{2,\mu}.
    \end{equation}
\end{lemma}
\begin{proof}
    By Jensen's inequality,
    \begin{equation}
        \E_\mu \big[|f(s,a)|\big] = \E_\mu \left[\sqrt{\big(f(s,a)\big)^2}\right] \leq \sqrt{\E_\mu [\big(f(s,a)\big)^2]} = \sqrt{\|f\|_{2,\mu}^2} = \|f\|_{2,\mu}.
    \end{equation}
\end{proof}

The next lemma is a simple fact about $f(d_0,\pi)$. The proof is omitted as it can be shown by a standard telescoping argument.
\begin{lemma}\label{lemma:f0}
Let $d_0$ be the initial state distribution of the MDP. 
For any policies $\pi,\tilde{\pi}$ and $f:\SS\times\AA\to\R$,
    \begin{equation}
    (1-\gamma) f(d_0, \pi) = \E_{\tilde{\pi}}[f(s,\pi) - \PP^\pi f(s,a)].
    \end{equation}
In particular, $\E_{\tilde{\pi}}[f(s,\pi) - \PP^\pi f(s,a)] = (1-\gamma) f(d_0, \pi) = \E_\pi[f(s,\pi) - \PP^\pi f(s,a)]$.
\end{lemma}

\subsection{Main Results}

We first show the approximate robust policy improvement property in the next lemma. This property provides a performance lower bound for the learned policy $\hat{\pi}$ relative to the behavioral policy $\mu$.
Note that, in the infinite sample regime, $\epsilon_\nu\to0$ and $\epsilon_\mu\to0$, we have robust policy improvement property for \emph{any} $\alpha\geq 0$ and $\beta\geq 0$.
\begin{proposition}[Robust Policy Improvement. General version of~\cref{prop:rpi-main-txt}]\label{prop:rpi-finite}
    Suppose that the policy class $\Pi$, critic class $\FF$ and reward class $\GG$ satisfy ~\cref{assumption:realizability}. Let $\hat\pi$ be the solution to~\eqref{eq:mahalo-finite}. Denote $\epsilon_\mu\coloneqq \frac{V_\text{max}^2 \log ( |\NN_\infty(\FF, V_\text{max} / |\DD_\AA|)| |\NN_\infty(\GG, R_\text{max} / |\DD_\AA|)| |\NN_{\infty,1}(\Pi, 1 / |\DD_\AA|)| / \delta)}{|\DD_\AA|}$ and $\epsilon_\nu\coloneqq \frac{R_\text{max}^2 \log ( \NN_\infty(\GG, R_\text{max} / |\DD_R|) / \delta)}{|\DD_R|}$. 
    Define 
    \begin{align}
         \epsilon_J \coloneqq \sqrt{\epsilon_\mu} + \alpha \epsilon_\nu + \beta \epsilon_\mu.
    \end{align}
    For any policy $\pi\in\Pi$,
    \begin{equation}
    J(\mu) - J(\hat\pi)  \leq \frac{1}{1-\gamma}\Big(- \LL_\mu(\pi, f^\pi)  + \OO(\epsilon_J)\Big),
    \end{equation}
    where $f^\pi$ is the solution to the followers' objective defined in~\eqref{eq:mahalo-finite}.
    In particular, with the choice of $\pi=\mu$, we have robust policy improvement: 
    \begin{equation}
        J(\mu) - J(\hat\pi) \leq \OO\left(\frac{\epsilon_J}{1-\gamma}\right). 
    \end{equation}
\end{proposition}
\begin{proof}
Since $(\hat{\pi}, f^{\hat\pi}, g^{\hat\pi})$ is the solution to the Stackelberg game~\eqref{eq:mahalo-finite}, 
\begin{align*}
    (1-\gamma) (J(\hat\pi) - J(\mu)) &= \E_\mu[Q^{\hat\pi}(s,\hat\pi) - Q^{\hat\pi}(s,a)] & \text{(Performance Difference Lemma)}\\
    &=\LL_\mu(\hat\pi, Q^{\hat\pi}) & \text{(Definition of $\LL_\mu$)}\\
    &\geq \LL_{\DD_\AA}(\hat\pi, Q^{\hat\pi}) + \alpha \EE_{\DD_R}(R) + \beta \EE_{\DD_\AA}(\hat\pi, Q^{\hat\pi}, R) \\
    &\qquad - \OO(\sqrt{\epsilon_\mu}) - \OO(\alpha \epsilon_\nu) - \OO(\beta \epsilon_\mu) & \text{(\cref{lemma:concentration-mu-q}, \cref{lemma:concentration-nu-r}, \cref{lemma:concentration-l})}\\
    &\geq \LL_{\DD_\AA}(\hat\pi, f^{\hat\pi}) + \alpha \EE_{\DD_R}(g^{\hat\pi}) + \beta \EE_{\DD_\AA}(\hat\pi, f^{\hat\pi}, g^{\hat\pi})\\
    &\qquad-  \OO(\sqrt{\epsilon_\mu} + \alpha \epsilon_\nu + \beta \epsilon_\mu) & \text{(Optimality of $(f^{\hat\pi},g^{\hat\pi})$,  $Q^{\hat\pi}\in\FF, R\in\GG$)}\\
    &\geq \LL_{\DD_\AA}(\hat{\pi}, f^{\hat\pi}) - \OO(\sqrt{\epsilon_\mu} + \alpha \epsilon_\nu + \beta \epsilon_\mu) &\text{($\EE_{\DD_R}(g^{\hat\pi})\geq 0$\text{ and} $\EE_{\DD_\AA}(\hat\pi, f^{\hat\pi}, g^{\hat\pi})\geq 0$)}\\
    &\geq \LL_{\DD_\AA}(\pi, f^{\pi}) - \OO(\sqrt{\epsilon_\mu} + \alpha \epsilon_\nu + \beta \epsilon_\mu) &\text{(Optimality of $\hat\pi)$}\\
    &\geq \LL_\mu(\pi, f^\pi)  - \OO(\sqrt{\epsilon_\mu}) - \OO(\sqrt{\epsilon_\mu} + \alpha \epsilon_\nu + \beta \epsilon_\mu) &\text{(\cref{lemma:concentration-l})}\\
    &= \LL_\mu(\pi, f^\pi)  - \OO(\sqrt{\epsilon_\mu} + \alpha \epsilon_\nu + \beta \epsilon_\mu).
\end{align*}
\end{proof}

The following lemma establishes an upper bound on the reward error $\EE_{\DD_R}$ and Bellman error $\EE_{\DD_\AA}$ for $f^\pi$ and $g^\pi$, the minimizer of the followers' objective in~\eqref{eq:mahalo-finite}. Intuitively, with sufficiently large $\alpha$ and $\beta$, $f^\pi$ and $g^\pi$ should induce small reward error $\EE_{\DD_R}(g^\pi)$ and Bellman error $\EE_{\DD_\AA}(\pi, f^\pi, g^\pi)$.
\begin{lemma}\label{lemma:regularize-error-bound}
    Assume~\cref{assumption:realizability} holds. 
    We have, for any $\pi\in\Pi$,
    \begin{equation}
        \alpha \EE_{\DD_R}(g^{\pi})  + \beta \EE_{\DD_\AA}(\pi, f^{\pi}, g^{\pi}) \leq 2V_\text{max} + \OO(\alpha \epsilon_\nu + \beta \epsilon_\mu),
    \end{equation}
    where $\epsilon_\mu$ and $\epsilon_\nu$ are defined in \cref{lemma:concentration-mu} and \cref{lemma:concentration-nu}, respectively.
    This implies that $\EE_{\DD_R}(g^{\pi}) \leq \OO\left(\frac{1}{\alpha} V_\text{max} + \epsilon_\nu + \frac\beta\alpha \epsilon_\mu\right)$ and $\EE_{\DD_\AA}(\pi, f^{\pi}, g^{\pi})\leq \OO\left(\frac{1}{\beta} V_\text{max} + \frac\alpha\beta\epsilon_\nu + \epsilon_\mu\right)$. 
\end{lemma}
\begin{proof}
    By~\cref{assumption:realizability}, $Q^{\pi}\in\FF$ and $R\in\GG$. By optimality of $f^{\pi}$ and $g^{\pi}$,
    \begin{equation}\label{eq:e-bound-proof}
        \begin{split}
            \LL_{\DD_\AA} ({\pi}, f^{{\pi}}) + \alpha \EE_{\DD_R}(g^{{\pi}}) + \beta \EE_{\DD_\AA}(\pi, f^{\pi}, g^{\pi})
            &\leq \LL_{\DD_\AA} ({\pi}, Q^{{\pi}}) + \alpha \EE_{\DD_R}(R) + \beta \EE_{\DD_\AA}(\pi, Q^{\pi}, R)\\
            &\leq \LL_{\DD_\AA} ({\pi}, Q^{{\pi}}) + \OO(\alpha \epsilon_\nu + \beta \epsilon_\mu).
        \end{split}
    \end{equation}
    The last inequality follows from \cref{lemma:concentration-mu} and \cref{lemma:concentration-nu}.

    By definition of $\LL_{\DD_\AA}$, we have
    $\LL_{\DD_\AA}({\pi}, f^{{\pi}}) = \E_{\DD_\AA}[f^{{\pi}} (s,\pi) - f^{{\pi}}(s,a)] \geq \E_\mu[0-V_\text{max}] = -V_\text{max}$. Similarly, $\LL_{\DD_\AA}(\hat{\pi}, Q^{\hat{\pi}}) = \E_{\DD_\AA}[Q^{\hat{\pi}} (s,\hat\pi) - Q^{\hat{\pi}}(s,a)] \leq \E_\mu[V_\text{max} - 0] = V_\text{max}$. By combining these with~\eqref{eq:e-bound-proof}, we have
    \begin{equation}
        \alpha \EE_\nu(g^{{\pi}}) + \beta \EE_\mu(\pi, f^{\pi}, g^{\pi}) \leq 2V_{\text{max}} + \OO(\alpha \epsilon_\nu + \beta \epsilon_\mu).
    \end{equation}
\end{proof}

We are now ready to prove the main theorem. It provides a performance lower bound for the learned policy $\hat{\pi}$ with respect to any comparator policy $\pi\in\Pi$.
\begin{theorem}[General version of~\cref{thm:main-txt}] \label{th:mahalo theorem (relative)}
Let $\hat{\pi}$ be the solution to the Stackelberg game~\eqref{eq:mahalo-finite} and let $\pi\in\Pi$ be any comparator policy. 
Let $C_1\geq 1, C_2\geq 1$ be any constants, $\rho\in \Delta (\SS\times \AA)$ be an arbitrary distribution that satisfies $\CC(\rho;\mu,\FF,\GG,{\pi})\leq C_1$ and $\CC(\rho;\nu,\GG)\leq C_2$. Let $\epsilon_\mu$ and $\epsilon_\nu$ be as defined in~\cref{lemma:concentration-mu} and~\cref{lemma:concentration-nu}, respectively. Choosing $\alpha = \Theta\left(\frac{V_\text{max}^{1/3} (\sqrt{C_1\epsilon_\nu} + \sqrt{C_2\epsilon_\mu})^{2/3}}{\epsilon_\nu}\right)$ and $\beta=\Theta\left(\frac{V_\text{max}^{1/3} (\sqrt{C_1\epsilon_\nu} + \sqrt{C_2\epsilon_\mu})^{2/3}}{\epsilon_\mu}\right)$, with high probability:
\begin{equation}
    \begin{split}
        J(\pi) - J(\hat{\pi}) &\leq \OO\left(\frac{(\sqrt{C_1\epsilon_\nu} + \sqrt{C_2\epsilon_\mu}) + V_\text{max}^{1/3} (\sqrt{C_1\epsilon_\nu} + \sqrt{C_2\epsilon_\mu})^{2/3}}{1-\gamma}\right)\\
        &\quad +\underbrace{\frac{\langle d^\pi \setminus \rho , \bar g^\pi + \PP^\pi f^\pi - f^\pi\rangle}{1-\gamma}}_{\text{off-support error (dynamics)}} + \underbrace{\frac{\langle (d^\pi \ominus \mu)\setminus\rho, |\bar R - \bar g^\pi| \rangle}{1-\gamma}}_{\text{off-support error (reward)}}
    \end{split}
\end{equation}
where $(d^\pi \ominus \mu) \coloneqq d^\pi \setminus \mu + \mu \setminus d^\pi$. See~\cref{lemma:distr-shift} for definitions of $\cdot\setminus\cdot$ and $\langle\cdot,\cdot\rangle$. 
\end{theorem}

\begin{proof}
    We have,
    \begin{align*}
        (1-\gamma) (J(\pi) - J(\hat{\pi})) &= (1-\gamma) (J(\pi) - J(\mu)) - (1-\gamma) (J(\hat\pi) - J(\mu))\\
        &\leq (1-\gamma) (J(\pi) - J(\mu)) - \LL_\mu(\pi, f^{\pi}) + \OO(\epsilon_J) & \text{(\cref{prop:rpi-finite})}\\
        &= \E_\pi[\bar R(s,a)] - \E_\mu[\bar R(s,a)] \\
        &\qquad - \E_\mu[f^\pi(s,\pi) - f^\pi(s,a)] + \OO(\epsilon_J) & \text{(Definition of $J$ and $\LL_\mu$)}\\
        &= \E_\pi[\bar R(s,a)] - \E_\mu[\bar R(s,a)] \\
        &\qquad - \E_\mu[f^\pi(s,\pi) - \PP^\pi f(s,a) + \PP^\pi f^\pi(s,a) - f^\pi(s,a) ] + \OO(\epsilon_J)\\
        &= \E_\pi[\bar R(s,a)] - \E_\mu[f^\pi(s,\pi) - \PP^\pi f^\pi(s,a)] \\
        &\qquad + \E_\mu[f^\pi(s,a) - \bar R(s,a) - \PP^\pi f^\pi(s,a)  ] + \OO(\epsilon_J)\\
        &= \E_\pi[\bar R(s,a)] - \E_\pi[f^\pi(s,\pi) - \PP^\pi f^\pi(s,a)] \\
        &\qquad + \E_\mu[(f^\pi- \bar R - \PP^\pi f^\pi)(s,a)] + \OO(\epsilon_J) & \text{(\cref{lemma:f0})}\\
        &= \E_\pi[(\bar R + \PP^\pi f^\pi - f^\pi)(s,a)]\\
        &\qquad +\E_\mu[(f^\pi- \bar R - \PP^\pi f^\pi)(s,a)] + \OO(\epsilon_J)\\
        &= \underbrace{\E_\pi[(\bar g^\pi + \PP^\pi f^\pi - f^\pi)(s,a)]}_{\text{(I)}} + \underbrace{\E_\mu[(f^\pi - \bar g^\pi - \PP^\pi f^\pi)(s,a)]}_{\text{(II)}} \\
        &\quad\quad + \underbrace{\E_\pi[(\bar R - \bar g^\pi)(s,a)]  + \E_\mu[(\bar g^\pi - \bar R)(s,a)]}_{\text{(III)}} + \OO(\epsilon_J)
    \end{align*}

We first establish an upper bound for term (II). We will use this later to bound for (I). By~\cref{lemma:l1-to-l2},
\begin{align*}
    \text{(II)} &\leq \|(f^\pi - \bar g^\pi - \PP^\pi f^\pi)(s,a)\|_{2,\mu} = \sqrt{\EE_\mu(\pi, f^{\pi}, g^{\pi})} \\
    &\leq \sqrt{\EE_{\DD_\AA}(\pi, f^{\pi}, g^{\pi})} + \OO(\sqrt{\epsilon_\mu}) &\text{(\cref{lemma:concentration-mu})}\\
    &\leq \OO\left(\sqrt{\frac{1}{\beta} V_{\text{max}} + \frac{\alpha}{\beta}\epsilon_\nu + \epsilon_\mu}\right) + \OO(\sqrt{\epsilon_\mu}) &\text{(\cref{lemma:regularize-error-bound})}\\
    &= \OO\left(\sqrt{\frac{1}{\beta} V_{\text{max}}} + \sqrt{\frac{\alpha}{\beta}\epsilon_\nu} + \sqrt{\epsilon_\mu}\right).
\end{align*}

We now bound term (I). By~\cref{lemma:distr-shift},
\begin{align}
    \text{(I)} &= \E_\pi[(\bar g^\pi + \PP^\pi f^\pi - f^\pi)(s,a)]\leq 2\E_\rho[|(\bar g^\pi + \PP^\pi f^\pi - f^\pi)(s,a)|] + \langle d^\pi \setminus \rho , \bar g^\pi + \PP^\pi f^\pi - f^\pi\rangle
\end{align}
where, by~\cref{lemma:l1-to-l2},
\begin{align*}
    \E_\rho[|(\bar g^\pi + \PP^\pi f^\pi - f^\pi)(s,a)|] &\leq \|\bar g^\pi + \PP^\pi f^\pi - f^\pi\|_{2,\rho}\\
    &\leq \sqrt{C_2 \|\bar g^\pi + \PP^\pi f^\pi - f^\pi\|_{2,\mu}^2}& \text{(Definition of $\CC(\rho;\mu,\FF,\GG,{\pi})$)}\\
    &\leq \OO\left(\sqrt{C_2}\left(\sqrt{\frac{1}{\beta} V_{\text{max}}} + \sqrt{\frac{\alpha}{\beta}\epsilon_\nu} + \sqrt{\epsilon_\mu}\right)\right). & \text{(\cref{lemma:regularize-error-bound})}
\end{align*}
Hence, we have
\begin{equation}
    \text{(I)} \leq \OO\left(\sqrt{C_2}\left(\sqrt{\frac{1}{\beta} V_{\text{max}}} + \sqrt{\frac{\alpha}{\beta}\epsilon_\nu} + \sqrt{\epsilon_\mu}\right)\right) + \langle d^\pi \setminus \rho , \bar g^\pi + \PP^\pi f^\pi - f^\pi\rangle
\end{equation}

Finally, we establish a bound for term (III) in a similar way.
\begin{align*}
    \text{(III)} &= \E_\pi[(\bar R - \bar g^\pi)(s,a)]  + \E_\mu[(\bar g^\pi - \bar R)(s,a)] \\
    &= \langle d^\pi \setminus \mu, \bar R - \bar g^\pi \rangle + \langle \mu \setminus d^\pi, \bar g^\pi - \bar R\rangle\\
    &\leq \langle d^\pi \ominus \mu, |\bar R - \bar g^\pi| \rangle & (|\bar R-\bar g^\pi| \geq \bar R - \bar g^\pi, |\bar R-\bar g^\pi| \geq \bar g^\pi - \bar R)\\
    &\leq 2 \E_\rho[|\bar R - \bar g^\pi|] + \langle (d^\pi \ominus \mu)\setminus\rho, |\bar R - \bar g^\pi| \rangle & \text{(\cref{lemma:distr-shift})}
\end{align*}
where $(d^\pi \ominus \mu) \coloneqq d^\pi \setminus \mu + \mu \setminus d^\pi$.

By~\cref{lemma:l1-to-l2}, we have,
\begin{align*}
    \E_\rho[|\bar R - \bar g^\pi|] &\leq \|\bar R - \bar g^\pi\|_{2,\rho}\\
    &\leq\sqrt{C_1 \| R -  g^\pi\|_{2,\nu}^2}= \sqrt{C_1}\sqrt{\EE_\nu(g^\pi)} & \text{(Definition of $\CC(\rho;\nu,\GG)$)}\\
    &\leq \sqrt{C_1}\big(\sqrt{\EE_{\DD_R}(g^\pi)} + \OO(\sqrt{\epsilon_\nu})\big) & \text{(\cref{lemma:concentration-nu})}\\
    &\leq \OO\left(\sqrt{C_1}\left(\sqrt{\frac{1}{\alpha} V_{\text{max}}} + \sqrt{\epsilon_\nu} + \sqrt{\frac{\beta}{\alpha}\epsilon_\mu}\right)\right) &\text{(\cref{lemma:regularize-error-bound})}
\end{align*}

Finally, we have
\begin{equation}
    \begin{split}
        (1-\gamma) (J(\pi) - J(\hat{\pi})) &\leq \OO\left(\sqrt{C_1}\left(\sqrt{\frac{1}{\alpha} V_{\text{max}}} + \sqrt{\epsilon_\nu} + \sqrt{\frac{\beta}{\alpha}\epsilon_\mu}\right) + \sqrt{C_2}\left(\sqrt{\frac{1}{\beta} V_{\text{max}}} + \sqrt{\frac{\alpha}{\beta}\epsilon_\nu} + \sqrt{\epsilon_\mu}\right) + \alpha \epsilon_\nu + \beta \epsilon_\mu\right)\\
        &\quad +\langle d^\pi \setminus \rho , \bar g^\pi + \PP^\pi f^\pi - f^\pi\rangle + \langle (d^\pi \ominus \mu)\setminus\rho, |\bar R - \bar g^\pi| \rangle
    \end{split}
\end{equation}

By choosing $\alpha = \Theta\left(\frac{V_\text{max}^{1/3} (\sqrt{C_1\epsilon_\nu} + \sqrt{C_2\epsilon_\mu})^{2/3}}{\epsilon_\nu}\right)$ and $\beta=\Theta\left(\frac{V_\text{max}^{1/3} (\sqrt{C_1\epsilon_\nu} + \sqrt{C_2\epsilon_\mu})^{2/3}}{\epsilon_\mu}\right)$, we have
\begin{equation}
    \begin{split}
        (1-\gamma) (J(\pi) - J(\hat{\pi})) &\leq \OO\left((\sqrt{C_1\epsilon_\nu} + \sqrt{C_2\epsilon_\mu}) + V_\text{max}^{1/3} (\sqrt{C_1\epsilon_\nu} + \sqrt{C_2\epsilon_\mu})^{2/3}\right)\\
        &\quad +\langle d^\pi \setminus \rho , \bar g + \PP^\pi f^\pi - f^\pi\rangle + \langle (d^\pi \ominus \mu)\setminus\rho, |\bar R - \bar g^\pi| \rangle.
    \end{split}
\end{equation}
\end{proof}
\section{Absolute Pessimism}\label{sec:pspi}

We can implement MAHALO based on absolute pessimism to optimize for $J(\pi)$, instead of $J(\pi)-J(\mu)$. This results in a realization of MAHALO based on PSPI~\cite{xie2021bellman}.
For simplicity of illustration, we assume that $d_0$ is known. 
A model-free version can be realized the solving the two-player game below: 
\begin{align}\label{eq:mahalo-finite (absolute)}
        \hat{\pi} &\in \argmax_{\pi\in\Pi}  (1-\gamma) f^\pi(d_0, \pi)\\
        \text{s.t.}\quad f^\pi &\in \argmin_{f\in\FF, g\in\GG}
        (1-\gamma) f(d_0, \pi) + \alpha \EE_{\DD_R} (g) + \beta \EE_{\DD_\AA}(\pi, f, g),\nonumber
\end{align}
with $\alpha\geq 0, \beta \geq 0$ being hyperparameters, and
\begin{align*}
    \EE_{\DD_R} (g) &\coloneqq \E_{\DD_R} \big[ \big(g(s,s') - r \big)^2 \big],\\
    \EE_{\DD_\AA}(\pi, f, g) &\coloneqq \E_{\DD_\AA}\big[\big(f(s,a) - g(s,s') - \gamma f(s',\pi)\big)^2\big]\label{eq:bellman-consistency-finite}- \min_{f'\in\FF}\E_{\DD_\AA}\big[\big(f'(s,a) - g(s,s') - \gamma f(s',\pi)\big)^2\big].\nonumber
\end{align*}
Compared with \eqref{eq:mahalo-finite}, the formulation in \eqref{eq:mahalo-finite (absolute)} replaces $\LL_{\DD_\AA} (\pi, f)$ with $(1-\gamma) f^\pi(s_0, \pi)$ which is a surrogate of $(1-\gamma)J(\pi)$. Below we adapt the proof of \cref{thm:main-txt} to give a guarantee on the absolute pessimism case.

\begin{theorem}[Absolute Pessimism Version of MAHALO] \label{th:mahalo theorem (absolute)}
Let $\hat{\pi}$ be the solution to the Stackelberg game~\eqref{eq:mahalo-finite (absolute)} and let $\pi\in\Pi$ be any comparator policy. 
Let $C_1\geq 1, C_2\geq 1$ be any constants, $\rho\in \Delta (\SS\times \AA)$ be an arbitrary distribution that satisfies $\CC(\rho;\mu,\FF,\GG,{\pi})\leq C_1$ and $\CC(\rho;\nu,\GG)\leq C_2$. Let $\epsilon_\mu$ and $\epsilon_\nu$ be as defined in~\cref{lemma:concentration-mu} and~\cref{lemma:concentration-nu}, respectively. Choosing $\alpha = \Theta\left(\frac{V_\text{max}^{1/3} (\sqrt{C_1\epsilon_\nu} + \sqrt{C_2\epsilon_\mu})^{2/3}}{\epsilon_\nu}\right)$ and $\beta=\Theta\left(\frac{V_\text{max}^{1/3} (\sqrt{C_1\epsilon_\nu} + \sqrt{C_2\epsilon_\mu})^{2/3}}{\epsilon_\mu}\right)$, with high probability:
\begin{equation}
    \begin{split}
        J(\pi) - J(\hat{\pi}) &\leq \OO\left(\frac{(\sqrt{C_1\epsilon_\nu} + \sqrt{C_2\epsilon_\mu}) + V_\text{max}^{1/3} (\sqrt{C_1\epsilon_\nu} + \sqrt{C_2\epsilon_\mu})^{2/3}}{1-\gamma}\right) + \underbrace{\frac{\langle d^\pi \setminus \rho , \bar R  + \PP^\pi f^\pi - f^\pi\rangle}{1-\gamma}}_{\text{off-support error }}
    \end{split}
\end{equation}
 See~\cref{lemma:distr-shift} for definitions of $\cdot\setminus\cdot$ and $\langle\cdot,\cdot\rangle$. 
\end{theorem}

If we compare \cref{th:mahalo theorem (absolute)} of absolute pessimism and \cref{th:mahalo theorem (relative)} of relative pessimism, we see that the main difference is in how the off-support error is measured. \cref{th:mahalo theorem (absolute)} measures the off-support error in $d^\pi \setminus \rho$ for both Bellman error and rewards (because  $\frac{\langle d^\pi \setminus \rho , \bar R  + \PP^\pi f^\pi - f^\pi\rangle}{1-\gamma} = 
\frac{\langle d^\pi \setminus \rho , \bar g  + \PP^\pi f^\pi - f^\pi\rangle}{1-\gamma} + \frac{\langle d^\pi \setminus \rho , \bar R  - \bar g \rangle}{1-\gamma}$), whereas \cref{th:mahalo theorem (relative)} uses $d^\pi \setminus \rho$ for Bellman error and $(d^\pi \ominus \mu)\setminus\rho$ for reward. 
As a result, the upper bound in \cref{th:mahalo theorem (relative)} considers how reward generalizes (off-support) on both both $d^\pi$ and $\mu$, but the one in \cref{th:mahalo theorem (absolute)} only concerns (one-sided) generalization on $d^\pi$, which is preferable  in some sense.
However, the policy learned by the absolute pessimism version does not have robust policy improvement guarantee. In experiments, we found this absolute pessimism version is more sensitive to hyperparameter choices than the MAHALO based on relative pessimism.



\subsection{Technical Lemmas}

\begin{proposition}[Absolute Pessimism Version of \cref{prop:rpi-finite}]\label{prop:ap-lower-bound (absolute)}
    Suppose that the policy class $\Pi$, critic class $\FF$ and reward class $\GG$ satisfy ~\cref{assumption:realizability}. Let $\hat\pi$ be the solution to~\eqref{eq:mahalo-finite}. Denote $\epsilon_\mu\coloneqq \frac{V_\text{max}^2 \log ( |\NN_\infty(\FF, V_\text{max} / |\DD_\AA|)| |\NN_\infty(\GG, R_\text{max} / |\DD_\AA|)| |\NN_{\infty,1}(\Pi, 1 / |\DD_\AA|)| / \delta)}{|\DD_\AA|}$ and $\epsilon_\nu\coloneqq \frac{R_\text{max}^2 \log ( \NN_\infty(\GG, R_\text{max} / |\DD_R|) / \delta)}{|\DD_R|}$. 
    Define 
    \begin{align}
         \epsilon_J \coloneqq  \alpha \epsilon_\nu + \beta \epsilon_\mu.
    \end{align}
    For any policy $\pi\in\Pi$,
    \begin{equation}
     f^\pi(d_0, \pi)  \leq J(\hat\pi)  + \OO(\epsilon_J),        
    \end{equation}
    where $f^\pi$ is the solution to the followers' objective defined in~\eqref{eq:mahalo-finite (absolute)}.
\end{proposition}
\begin{proof}
Since $(\hat{\pi}, f^{\hat\pi}, g^{\hat\pi})$ is the solution to the Stackelberg game~\eqref{eq:mahalo-finite},
\begin{align*}
     J(\hat\pi) 
    &\geq  Q^{\hat\pi}(d_0, \hat\pi) + \alpha \EE_{\DD_R}(R) + \beta \EE_{\DD_\AA}(\hat\pi, Q^{\hat\pi}, R) \\
    &\qquad - \OO(\alpha \epsilon_\nu) - \OO(\beta \epsilon_\mu) & \text{(\cref{lemma:concentration-mu-q}, \cref{lemma:concentration-nu-r})}\\
    &\geq f^{\hat\pi}(d_0,\hat\pi) + \alpha \EE_{\DD_R}(g^{\hat\pi}) + \beta \EE_{\DD_\AA}(\hat\pi, f^{\hat\pi}, g^{\hat\pi})\\
    &\qquad-  \OO( \alpha \epsilon_\nu + \beta \epsilon_\mu) & \text{(Optimality of $(f^{\hat\pi},g^{\hat\pi})$,  $Q^{\hat\pi}\in\FF, R\in\GG$)}\\
    &\geq f^{\hat\pi}(d_0,\hat\pi) - \OO( \alpha \epsilon_\nu + \beta \epsilon_\mu) &\text{($\EE_{\DD_R}(g^{\hat\pi})\geq 0$\text{ and} $\EE_{\DD_\AA}(\hat\pi, f^{\hat\pi}, g^{\hat\pi})\geq 0$)}\\
    &\geq f^{\pi}(d_0,\pi) - \OO(\alpha \epsilon_\nu + \beta \epsilon_\mu) &\text{(Optimality of $\hat\pi)$}
\end{align*}
\end{proof}

\begin{lemma}[Absolute Pessimism Version of \cref{lemma:regularize-error-bound}]\label{lemma:regularize-error-bound (absolute)}
    Assume~\cref{assumption:realizability} holds. 
    We have, for any $\pi\in\Pi$,
    \begin{equation}
        \alpha \EE_{\DD_R}(g^{\pi})  + \beta \EE_{\DD_\AA}(\pi, f^{\pi}, g^{\pi}) \leq V_\text{max} + \OO(\alpha \epsilon_\nu + \beta \epsilon_\mu),
    \end{equation}
    where $\epsilon_\mu$ and $\epsilon_\nu$ are defined in \cref{lemma:concentration-mu} and \cref{lemma:concentration-nu}, respectively.
    This implies that $\EE_{\DD_R}(g^{\pi}) \leq \OO\left(\frac{1}{\alpha} V_\text{max} + \epsilon_\nu + \frac\beta\alpha \epsilon_\mu\right)$ and $\EE_{\DD_\AA}(\pi, f^{\pi}, g^{\pi})\leq \OO\left(\frac{1}{\beta} V_\text{max} + \frac\alpha\beta\epsilon_\nu + \epsilon_\mu\right)$. 
\end{lemma}
\begin{proof}
    By~\cref{assumption:realizability}, $Q^{\pi}\in\FF$ and $R\in\GG$. By optimality of $f^{\pi}$ and $g^{\pi}$,
    \begin{equation}\label{eq:e-bound-proof-abs}
        \begin{split}
             f^{{\pi}}(d_0, \pi)  + \alpha \EE_{\DD_R}(g^{{\pi}}) + \beta \EE_{\DD_\AA}(\pi, f^{\pi}, g^{\pi})
            &\leq Q^{{\pi}}(d_0, \pi) + \alpha \EE_{\DD_R}(R) + \beta \EE_{\DD_\AA}(\pi, Q^{\pi}, R)\\
            &\leq Q^{{\pi}}(d_0, \pi) + \OO(\alpha \epsilon_\nu + \beta \epsilon_\mu).
        \end{split}
    \end{equation}
    The last inequality follows from \cref{lemma:concentration-mu} and \cref{lemma:concentration-nu}.
    Since $Q^{{\pi}}(d_0, \pi)\leq V_{\text{max}}$ and $f^{{\pi}}(d_0, \pi)\geq 0$, we have
    \begin{equation}
        \alpha \EE_\nu(g^{{\pi}}) + \beta \EE_\mu(\pi, f^{\pi}, g^{\pi}) \leq V_{\text{max}} + \OO(\alpha \epsilon_\nu + \beta \epsilon_\mu).
    \end{equation}
\end{proof}

\subsection{Proof of \cref{th:mahalo theorem (absolute)}}
\begin{proof}
    We have,
    \begin{align*}
        (1-\gamma) (J(\pi) - J(\hat{\pi})) 
        &\leq (1-\gamma) J(\pi) - (1-\gamma) f^{\pi}(d_0, \pi) + \OO(\epsilon_J) & \text{(\cref{prop:ap-lower-bound (absolute)})}\\
        &= \E_\pi[\bar R(s,a)]  - (1-\gamma) f^{\pi}(d_0, \pi) + \OO(\epsilon_J) & \text{(Definition of $J$ and $\LL_\mu$)}\\
        &= \E_\pi[\bar R(s,a)]  - \E_\pi[f^\pi(s,\pi) - \PP^\pi f(s,a) ] + \OO(\epsilon_J)   & \text{(\cref{lemma:f0})}\\
        &= \E_\pi[(\bar R + \PP^\pi f^\pi - f^\pi)(s,a)] + \OO(\epsilon_J)\\
        &= \underbrace{\E_\pi[(\bar g^\pi + \PP^\pi f^\pi - f^\pi)(s,a)]}_{\text{(I)}}  + \underbrace{\E_\pi[(\bar R - \bar g^\pi)(s,a)]}_{\text{(II)}} + \OO(\epsilon_J)
    \end{align*}


We bound term (I). By~\cref{lemma:distr-shift},
\begin{align}
    \text{(I)} &= \E_\pi[(\bar g^\pi + \PP^\pi f^\pi - f^\pi)(s,a)]\leq 2\E_\rho[|(\bar g^\pi + \PP^\pi f^\pi - f^\pi)(s,a)|] + \langle d^\pi \setminus \rho , \bar g^\pi + \PP^\pi f^\pi - f^\pi\rangle
\end{align}
where, by~\cref{lemma:l1-to-l2},
\begin{align*}
    \E_\rho[|(\bar g^\pi + \PP^\pi f^\pi - f^\pi)(s,a)|] &\leq \|\bar g^\pi + \PP^\pi f^\pi - f^\pi\|_{2,\rho}\\
    &\leq \sqrt{C_2 \|\bar g^\pi + \PP^\pi f^\pi - f^\pi\|_{2,\mu}^2}& \text{(Definition of $\CC(\rho;\mu,\FF,\GG,{\pi})$)}\\
    &\leq \sqrt{C_2 \|\bar g^\pi + \PP^\pi f^\pi - f^\pi\|_{2,\DD_A}^2 + C_2 \epsilon_\mu}& \text{(Definition of $\CC(\rho;\mu,\FF,\GG,{\pi})$)}\\
    &\leq \OO\left(\sqrt{C_2}\left(\sqrt{\frac{1}{\beta} V_{\text{max}}} + \sqrt{\frac{\alpha}{\beta}\epsilon_\nu} + \sqrt{\epsilon_\mu}\right)\right). & \text{(\cref{lemma:regularize-error-bound (absolute)})}
\end{align*}
Hence, we have
\begin{equation}
    \text{(I)} \leq \OO\left(\sqrt{C_2}\left(\sqrt{\frac{1}{\beta} V_{\text{max}}} + \sqrt{\frac{\alpha}{\beta}\epsilon_\nu} + \sqrt{\epsilon_\mu}\right)\right) + \langle d^\pi \setminus \rho , \bar g^\pi + \PP^\pi f^\pi - f^\pi\rangle
\end{equation}

Finally, we establish a bound for term (II) in a similar way.
\begin{align*}
    \text{(II)} &= \E_\pi[(\bar R - \bar g^\pi)(s,a)]   \\
    &\leq 2 \E_\rho[|\bar R - \bar g^\pi|] + \langle d^\pi \setminus\rho, \bar R - \bar g^\pi \rangle & \text{(\cref{lemma:distr-shift})}
\end{align*}
By~\cref{lemma:l1-to-l2}, we have,
\begin{align*}
    \E_\rho[|\bar R - \bar g^\pi|] &\leq \|\bar R - \bar g^\pi\|_{2,\rho}\\
    &\leq\sqrt{C_1 \| R -  g^\pi\|_{2,\nu}^2}= \sqrt{C_1}\sqrt{\EE_\nu(g^\pi)} & \text{(Definition of $\CC(\rho;\nu,\GG)$)}\\
    &\leq \sqrt{C_1}\big(\sqrt{\EE_{\DD_R}(g^\pi)} + \OO(\sqrt{\epsilon_\nu})\big) & \text{(\cref{lemma:concentration-nu})}\\
    &\leq \OO\left(\sqrt{C_1}\left(\sqrt{\frac{1}{\alpha} V_{\text{max}}} + \sqrt{\epsilon_\nu} + \sqrt{\frac{\beta}{\alpha}\epsilon_\mu}\right)\right) &\text{(\cref{lemma:regularize-error-bound (absolute)})}
\end{align*}

Finally, we have
\begin{equation}
    \begin{split}
        (1-\gamma) (J(\pi) - J(\hat{\pi})) 
        &\leq \OO\left(\sqrt{C_1}\left(\sqrt{\frac{1}{\alpha} V_{\text{max}}} + \sqrt{\epsilon_\nu} + \sqrt{\frac{\beta}{\alpha}\epsilon_\mu}\right) + \sqrt{C_2}\left(\sqrt{\frac{1}{\beta} V_{\text{max}}} + \sqrt{\frac{\alpha}{\beta}\epsilon_\nu} + \sqrt{\epsilon_\mu}\right) + \alpha \epsilon_\nu + \beta \epsilon_\mu\right)\\
        &\quad +\langle d^\pi \setminus \rho , \bar g^\pi + \PP^\pi f^\pi - f^\pi\rangle + \langle d^\pi \setminus\rho, \bar R - \bar g^\pi \rangle\\
        &= \OO\left(\sqrt{C_1}\left(\sqrt{\frac{1}{\alpha} V_{\text{max}}} + \sqrt{\epsilon_\nu} + \sqrt{\frac{\beta}{\alpha}\epsilon_\mu}\right) + \sqrt{C_2}\left(\sqrt{\frac{1}{\beta} V_{\text{max}}} + \sqrt{\frac{\alpha}{\beta}\epsilon_\nu} + \sqrt{\epsilon_\mu}\right) + \alpha \epsilon_\nu + \beta \epsilon_\mu\right)\\
        &\quad +\langle d^\pi \setminus \rho , \bar R  + \PP^\pi f^\pi - f^\pi\rangle
    \end{split}
\end{equation}

By choosing $\alpha = \Theta\left(\frac{V_\text{max}^{1/3} (\sqrt{C_1\epsilon_\nu} + \sqrt{C_2\epsilon_\mu})^{2/3}}{\epsilon_\nu}\right)$ and $\beta=\Theta\left(\frac{V_\text{max}^{1/3} (\sqrt{C_1\epsilon_\nu} + \sqrt{C_2\epsilon_\mu})^{2/3}}{\epsilon_\mu}\right)$, we have
\begin{equation}
    \begin{split}
        (1-\gamma) (J(\pi) - J(\hat{\pi})) &\leq \OO\left((\sqrt{C_1\epsilon_\nu} + \sqrt{C_2\epsilon_\mu}) + V_\text{max}^{1/3} (\sqrt{C_1\epsilon_\nu} + \sqrt{C_2\epsilon_\mu})^{2/3}\right)\\
        &\quad +\langle d^\pi \setminus \rho , \bar R + \PP^\pi f^\pi - f^\pi\rangle 
    \end{split}
\end{equation}
\end{proof}

\subsection{Experimental Results on D4RL Benchmark}
We implement MAHALO realized by PSPI~\citep{xie2021bellman}
which we refer to as MAHALO-PSPI, and test it on the D4RL benchmark. We observe that MAHALO-PSPI has similar performance as MAHALO-ATAC. MAHALO-PSPI achieves top performance in all hopper and walker tasks. It does slighly worse than MAHALO-ATAC in halfcheetah tasks. We hypothize that this is a limitation of the base algorithm PSPI, since it achieves lower score than ATAC when doing offline RL on halfcheetah datasets~\cite{cheng2022adversarially}. For MAHALO-PSPI, we use a \emph{fixed} ratio of $(\alpha/\beta)\equiv 10000.0$ and tune $\beta\in[10.0, 100.0, 1000.0]$. For other hyperparameters, we use what is described in~\cref{sec:exp-details}.

\input{tables/abs_pess_d4rl_results}
}

\rev{
\section{Experimental Results with Standard Error}\label{sec:std}

We report the average scores and standard errors across random seeds for D4RL and Meta-World in~\cref{tab:mujoco-std} and~\cref{tab:mw-rewards-std}, respectively. MAHALO is the overall best-performing algorithm across the board, and its standard errors are in general small. Here we additionally consider two baseline algorithms. First, we implement a variant of UDS~\cite{yu2022leverage} which we call \textbf{UDS-A}. UDS-A has two copies of labeled transitions, one labeled with true reward, while the other (falsely) labeled with minimum reward. 
UDS requires more information than MAHALO: UDS needs to know the common transitions between reward and dynamics datasets. UDS-A, hence, would be a more fair to compare with MAHALO as they make the same assumption on datasets.
Second, we consider a version of SMODICE~\cite{ma2022versatile} using $\chi^2$ divergence, which we refer to as \textbf{SMODICE-$\chi^2$}. We observe that UDS-A performs similarly with UDS, and hence underperforms MAHALO. SMODICE-$\chi^2$ achieves top performance (slighly better, but comparable to MAHALO) on Meta-World IL and ILfO tasks, potentially thanks to the training stability from $\chi^2$ divergence. But it fails miserably on all D4RL tasks. 

\input{tables/d4rl_results_with_std}
\input{tables/mw_results_with_std}
}
\end{document}

%% file: shortcuts.tex
\usepackage{amsmath}
\usepackage{amsfonts}
\usepackage{amssymb}
\usepackage{amsthm}
\usepackage{bm}
\usepackage{bbm}
\usepackage{mathtools}
\usepackage{enumitem}
\usepackage{thmtools,thm-restate}
\usepackage{algorithm}
\usepackage{algorithmic}
\usepackage{color}
\usepackage{graphicx}
\usepackage{comment}

\newcommand{\rev}[1]{{\leavevmode\color{black}#1}}

\newcommand{\highlight}[1]{{\leavevmode\color{magenta}#1}}

\def\AA{\mathcal{A}}\def\CC{\mathcal{C}}
\def\DD{\mathcal{D}}\def\EE{\mathcal{E}}\def\FF{\mathcal{F}}
\def\GG{\mathcal{G}}
\def\JJ{\mathcal{J}}\def\LL{\mathcal{L}}
\def\MM{\mathcal{M}}\def\NN{\mathcal{N}}\def\OO{\mathcal{O}}
\def\PP{\mathcal{P}}
\def\SS{\mathcal{S}}

\def\Ebb{\mathbb{E}}

\def\Rbb{\mathbb{R}}

\def\R{\Rbb}
\def\one{{\mathbbm1}}

\def\*{\star}
\DeclareMathSymbol{\mhef}{\mathord}{operators}{`\-}

\DeclareMathOperator*{\argmin}{arg\,min}
\DeclareMathOperator*{\argmax}{arg\,max}

\newcommand{\E}{\Ebb}

%% file: tables/formulation.tex
\begin{table*}[t]
\caption{Different problem formulations for sequential decision making on offline datasets. Our PLfO formulation is the most general and can leverage the broadest range of data, which makes it the most realistic. The other formulations can be reduced to PLfO with additional restrictions on data. * denotes that data can only be used partially, with either action or reward removed.}
\label{tab:formulation}
\begin{center}
\begin{small}
\begin{tabular}{c|c|c|c|c|c}
                & $(s,a,r,s')$ & $(s,a,s')$ & Expert $(s,a,s')$ & Expert $(s,s')$ & Non-expert $(s,r,s')$\\
\hline
Offline IL
                & \xmark*   & \cmark     & \cmark            & \xmark          & \xmark\\
\hline
Offline ILfO
                & \xmark*   & \cmark     & \xmark*           & \cmark          & \xmark\\
\hline
Offline RL
                & \cmark    & \xmark     & \xmark            & \xmark          & \xmark\\
\hline
Offline RL w/ Unlabeled Data
                & \cmark    & \cmark     & \xmark            & \xmark          & \xmark\\
\hline
Offline PLfO
(Proposed)          & \cmark    & \cmark     & \cmark            & \cmark          & \cmark\\
\end{tabular}
\end{small}
\end{center}
\vskip -0.25in
\end{table*}

%% file: algo.tex
\begin{algorithm*}[th]
   \caption{MAHALO (realized by ATAC)}
   \label{alg:mahalo-atac}
\begin{algorithmic}[1]
    \STATE {\bfseries Input:} Batch datasets \highlight{$\DD_R$}, $\DD_\AA$; policy $\pi$, critics $f_1$, $f_2$; coefficients \highlight{$\alpha$}, $\beta\geq 0$ and $\tau, w\in[0,1]$.
    \STATE Initialize target networks $\Bar{f}_1\gets f_1$, $\Bar{f}_2\gets f_2$.
    \FOR{$k=1,2,\ldots,K$}
        \STATE{Sample minibatches  \highlight{$\DD_{R}^{\text{mini}}$} and $\DD_{\AA}^{\text{mini}}$ from \highlight{$\DD_R$} and $\DD_\AA$. }
        \STATE{Compute critic loss $l_\text{critic}(f_i) \gets \LL_{D_\AA^\text{mini}}(\pi, f_i) + \beta \EE_{D^\text{mini}_\AA}^w (\pi, f_i, \highlight{g})$, for $i\in\{1,2\}$}
        \STATE{Compute reward loss \highlight{$l_\text{reward}(g) \gets \alpha \EE_{\DD_R^\text{mini}}(g) + \beta \sum_{i=\{1,2\}}\EE_{D^\text{mini}_\AA}^w (\pi, f_i, {g})$.}}
        \STATE{Update critic network $f_i\gets\text{Proj}_\FF (f_i-\eta_\text{fast} \nabla l_\text{critic})$ for $i\in\{1,2\}$.}
        \STATE{Update reward network \highlight{$g\gets\text{Proj}_\GG (g - \eta_\text{fast} \nabla l_\text{reward})$.}}
        \STATE{Compute actor loss $l_\text{actor}\gets -\LL_{\DD_\AA^\text{mini}}(\pi, f_1)$.}
        \STATE{Update actor network $\pi\gets \text{Proj}_\Pi(\pi - \eta_\text{slow} \nabla l_\text{actor})$}
        \STATE{Update target $\Bar{f}\gets (1-\tau) \Bar{f} + \tau f$ for $(f,\Bar{f})\in\{(f_i,\Bar{f}_i)\}_{i=1,2}$.}
    \ENDFOR
\end{algorithmic}
\end{algorithm*}

%% file: tables/new_d4rl_results.tex
\begin{table*}[t]
\caption{Results on D4RL benchmark~\cite{fu2020d4rl}. We show the average normalized score over 50 evaluation trials across 10 random seeds. (The standard errors are reported in~\cref{tab:mujoco-std}). Algorithms with scores greater than 90\% of the best score (excluding Oracle) are in bold. 
$^\dagger$ ATAC only uses data with both dynamics and reward information. $^+$ Oracle has access to reward for all dynamics data.}
\label{tab:mujoco}
\begin{center}
\begin{small}
\begin{tabular}{c|c|c|c|c|c|c|c|c|c||c}
\toprule
Scenario & Dataset & MAHALO & RP & AP & UDS 
& ATAC$^\dagger$ & BCO & BC & {\tiny SMODICE} & Oracle$^+$ \\ 
\hline
ILfO & hopper & \textbf{104.66} & \textbf{97.48} & 45.97 & - 
& - & 46.80 & - & 67.72 & -\\ 
 & walker & \textbf{88.60} & 77.15 & 61.11 & - 
 & - & 63.02 & - & 1.52 & -\\ 
 & halfcheetah & \textbf{61.24} & 36.00 & 4.87 & - 
 & - & 5.16 & - & \textbf{59.64} & -\\ 
\hline
IL & hopper & \textbf{104.06} & \textbf{97.88} & 53.35 & 32.21 
& 63.56 & - & 32.12 & 74.75 & -\\ 
 & walker & \textbf{89.03} & 77.71 & 63.53 & 8.45 
 & 78.35 & - & 18.78 & 0.94 & -\\ 
 & halfcheetah & \textbf{54.99} & 23.20 & 3.74 & 25.82 
 & 3.64 & - & 22.36 & \textbf{58.40} & -\\ 
\hline
RLfO & hopper & \textbf{106.47} & \textbf{105.65} & 47.01 & - 
& - & - & - & - &103.39 \\ 
 & walker & \textbf{96.65} & \textbf{97.26} & 63.30 & - 
 & - & - & - & - & 98.52 \\ 
 & halfcheetah & 50.38 & \textbf{68.66} & 3.35 & - 
 & - & - & - & - & 63.57 \\ 
\hline
RL-expert & hopper & 87.73 & \textbf{105.56} & 51.54 & \textbf{98.63} 
& 65.45 & - & - & - & 103.39 \\ 
 & walker & \textbf{103.18} & \textbf{98.31} & 56.27 & 72.97 
 & 66.40 & - & - & - &98.52 \\ 
 & halfcheetah & 48.43 & \textbf{64.37} & 3.47 & 13.68 
 & 3.41 & - & - & - & 63.57 \\ 
\hline
RL-sample & hopper & \textbf{103.08} & \textbf{101.66} & 71.92 & 0.95 
& 71.50 & - & -  & - & 103.34 \\ 
 & walker & \textbf{95.00} & \textbf{94.59} & 5.94 & 0.00 
 & 0.48 & - & - & - & 95.73 \\ 
 & halfcheetah & \textbf{68.30} & \textbf{68.71} & 13.26 & 20.36 
 & 19.38 & - & - & - & 69.91 \\ 
\bottomrule
\end{tabular}
\end{small}
\end{center}
\vskip -0.2in
\end{table*}

%% file: tables/scenarios.tex
\begin{table}[t]
\caption{Five scenarios of PLfO considered in our experiments.}
\label{tab:scenarios}
\begin{center}
\begin{small}
\begin{tabular}{c|l|l}
\toprule
Scenarios & Mixed Quality Data & Expert \\
\hline
ILfO & state + action & state\\
\hline
IL   & state + action & state + action\\
\hline
RLfO & state + action & state + reward\\
\hline
RL-expert & state + action & state + action \\
& & + reward\\
\hline
RL-sample & state + action +  & -\\
& (sampled trajs) reward & \\
\bottomrule
\end{tabular}
\end{small}
\end{center}
\vskip -7mm
\end{table}

%% file: tables/new_mw_results.tex
\begin{table*}[t]
\caption{Success rate of the final policy (with the exception of SMODICE$^*$) on Meta-World~\cite{yu2020meta}. The success rate is computed over 50 evaluation episodes. We report the average success rate across 10 random seeds. (The standard errors across random seeds are reported in~\cref{tab:mw-rewards-std}). We consider an episode success if it is able to reach the goal within $128$ steps. 
$^*$ SMODICE 
often diverges during training; we therefore take the success rate of its best performing policy during training instead of the final one.
}
\label{tab:mw-rewards}
\begin{center}
\begin{small}
\begin{tabular}{c|c|c|c|c|c|c|c|c|c||c}
\toprule
Scenario & Dataset & MAHALO & RP & AP & UDS 
& ATAC$^\dagger$ & BCO & BC & {\tiny SMODICE}$^*$ & Oracle$^+$ \\ 
\hline
ILfO & reach & \textbf{65.0} & \textbf{62.6} & 13.0 & - 
& - & 11.6 & - & 10.6 & - \\ 
 & push & \textbf{62.4} & 11.6 & 12.4 & - 
 & - & 14.6 & - & 0.4 & - \\ 
 & plate-slide & \textbf{100.0} & 22.0 & \textbf{94.0} & - 
 & - & 75.4 & - & 0.0 & - \\ 
 & handle-press & 75.4 & 32.4 & \textbf{96.6} & - 
 & - & \textbf{87.8} & - & 16.6 & - \\ 
 & button-press & \textbf{100.0} & \textbf{100.0} & \textbf{93.6} & - 
 & - & \textbf{93.8} & - & 0.4 & - \\ 
\hline
IL & reach & 24.6 & 23.6 & 38.0 & 21.6 
 & \textbf{98.0} & - & 62.2 & 19.8 &  - \\ 
 & push & \textbf{92.6} & 11.8 & 35.2 & 79.2 
 & 50.0 & - & \textbf{91.4} & 0.2 & - \\ 
 & plate-slide & 80.4 & 34.4 & \textbf{89.8} & 76.2 
 & \textbf{89.4} & - & \textbf{85.3} & 0.0 & - \\ 
 & handle-press & 71.4 & 34.8 & \textbf{100.0} & 35.2 
 & \textbf{97.0} & - & 75.2 & 20.2 & - \\ 
 & button-press & \textbf{100.0} & \textbf{99.8} & \textbf{96.2} & \textbf{100.0} 
 & \textbf{99.6} & - & \textbf{100.0} & 0.2 & - \\ 
\hline
RLfO & reach & \textbf{86.4} & \textbf{88.0} & 15.0 & - 
 & - & - & - & - & 51.6\\ 
 & push & \textbf{58.2} & 32.0 & 20.2 & - 
 & - & - & - & - & 91.8 \\ 
 & plate-slide & \textbf{100.0} & \textbf{100.0} & 83.2 & - 
 & - & - & - & - & 100.0\\ 
 & handle-press & 77.8 & 84.4 & \textbf{96.0} & - 
 & - & - & - & - & 78.6 \\ 
 & button-press & \textbf{100.0} & \textbf{100.0} & \textbf{92.6} & - 
 & - & - & - & - & 100.0 \\ 
\hline
RL-expert & reach & 39.2 & 54.0 & 42.0 & 57.8 
& \textbf{98.4} & - & - & - & 51.6 \\ 
 & push & \textbf{95.6} & 88.6 & 40.4 & \textbf{90.4} 
 & \textbf{99.6} & - & - & - & 91.8 \\ 
 & plate-slide & \textbf{100.0} & \textbf{100.0} & 89.6 & \textbf{99.4} 
 & 85.4 & - & - & - & 100.0 \\ 
 & handle-press & 72.6 & 82.0 & \textbf{97.6} & 81.2 
 & \textbf{99.0} & - & - & - & 78.6 \\ 
 & button-press & \textbf{100.0} & \textbf{100.0} & \textbf{97.0} & \textbf{100.0} & \textbf{100.0} 
 & - & - & - & 100.0 \\ 
\hline
RL-sample & reach & \textbf{86.6} & \textbf{87.4} & 63.0 & \textbf{87.4} 
 & \textbf{85.4} & - & - & - & 88.8 \\ 
 & push & 40.6 & \textbf{47.8} & 29.2 & 35.8 
 & 35.0 & - & - & - & 46.0 \\ 
 & plate-slide & \textbf{100.0} & \textbf{100.0} & \textbf{97.6} & \textbf{100.0} 
 & \textbf{99.6} & - & - & - & 100.0 \\ 
 & handle-press & 78.4 & 81.0 & \textbf{100.0} & 76.8 
 & 82.8 & - & - & - & 83.4 \\ 
 & button-press & \textbf{100.0} & \textbf{100.0} & \textbf{100.0} & \textbf{100.0} 
 & \textbf{100.0} & - & - & - & 100.0 \\ 
 \bottomrule
\end{tabular}
\end{small}
\end{center}
\vskip -0.2in
\end{table*}

%% file: tables/abs_pess_d4rl_results.tex
\begin{table*}[t]
\caption{Results of MAHALO-PSPI on D4RL benchmark~\cite{fu2020d4rl} We show the average normalized score over 50 evaluation trials across 10 random seeds. We paste the scores of MAHALO-ATAC (which is presented in the main text) from~\cref{tab:mujoco} for a comparison with MAHALO-PSPI. MAHALO-PSPI shows comparable performance as MAHALO-ATAC. MAHALO-PSPI results are generated using a fixed $\alpha/\beta$ ratio of 10000.0.}
\label{tab:mujoco-abs}
\begin{center}
\begin{small}
\begin{tabular}{c|c|c|c}
\toprule
Scenario & Dataset & MAHALO-PSPI & MAHALO-ATAC\\ 
\hline
ILfO & hopper & 108.71 & 104.66\\ 
 & walker & 92.86 & 88.60\\ 
 & halfcheetah & 29.89 & 61.24\\ 
\hline
IL & hopper & 106.36 & 104.06\\ 
 & walker & 96.02 & 89.03\\ 
 & halfcheetah & 51.81 & 54.99\\ 
\hline
RLfO & hopper & 101.81 & 106.47\\ 
 & walker & 98.70 & 96.65\\ 
 & halfcheetah & 30.93 & 50.38\\ 
\hline
RL-expert & hopper & 98.09 & 87.73\\ 
 & walker & 91.07 & 103.18\\ 
 & halfcheetah & 27.15 & 48.43\\ 
\hline
RL-sample & hopper & 104.04 & 103.08\\ 
 & walker & 95.66 & 95.00\\ 
 & halfcheetah & 28.58 & 68.30\\ 
\bottomrule
\end{tabular}
\end{small}
\end{center}
\end{table*}

%% file: tables/d4rl_results_with_std.tex
\begin{table}[t]
\begin{scriptsize}
\hspace{0.25\textwidth}
\begin{adjustbox}{angle=90}
\begin{minipage}[c][\textwidth][s]{\textheight}
\caption{Results on D4RL benchmark~\cite{fu2020d4rl} We show the average normalized scores over 50 evaluation trials and the standard errors across 10 random seeds. Algorithms with scores greater than 0.9 of the best score (excluding Oracle) is marked in bold. MAHALO achieves top performance in almost every scenario and task except halfcheetah. MAHALO is also able to match Oracle performance, despite having access to less reward data. $^\dagger$ ATAC only uses data with both dynamics and reward information. $^+$ Oracle has access to reward for all dynamics data.}
\label{tab:mujoco-std}
\vspace{5mm}
\begin{tabular}{c|c|c|c|c|c|c|c|c|c|c|c||c}
\toprule
Scenario & Dataset & MAHALO & RP & AP & UDS &  UDS-A & ATAC$^\dagger$ & BCO & BC & { SMODICE} & SMODICE-$\chi^2$ & Oracle$^+$ \\ 
\hline
ILfO & hopper & \textbf{104.66$\pm$2.38} & \textbf{97.48$\pm$0.29} & 45.97$\pm$4.62 & - & - & - & 46.98$\pm$6.04 & - & 67.72 $\pm$ 4.19 & 4.12$\pm$3.00 & - \\ 
 & walker & \textbf{88.60$\pm$0.89} & 77.15$\pm$1.11 & 61.11$\pm$5.55 & - & - & - & 60.48$\pm$4.30 & - & 1.52$\pm$0.73 & 21.89$\pm$6.35 &  - \\ 
 & halfcheetah & \textbf{61.24$\pm$1.14} & 36.00$\pm$6.01 & 4.87$\pm$1.02 & - & - & - & 5.53$\pm$0.95 & - & \textbf{59.64$\pm$1.17} & 37.96$\pm$1.46 & - \\ 
\hline
IL & hopper & \textbf{104.06$\pm$2.37} & \textbf{97.88$\pm$0.16} & 53.35$\pm$6.30 & 32.21$\pm$2.17 & 28.53$\pm$1.90 & 63.56$\pm$6.37 & - & 30.83$\pm$4.34 & 74.75$\pm$4.19 & 0.97$\pm$0.01 & - \\ 
 & walker & \textbf{89.03$\pm$1.04} & 77.71$\pm$0.81 & 63.53$\pm$3.43 & 8.45$\pm$4.35 & 14.56$\pm$6.52 & 78.35$\pm$4.08 & - & 14.96$\pm$5.63 & 0.94 $\pm$ 0.36 & 71.66$\pm$2.50 & - \\ 
 & halfcheetah & \textbf{54.99$\pm$5.63} & 23.20$\pm$4.53 & 3.74$\pm$1.06 & 25.82$\pm$2.87 & 28.46$\pm$2.03 & 3.64$\pm$0.95 & - & 26.87$\pm$2.73 & \textbf{58.40$\pm$1.00} & 39.63$\pm$2.55 &  - \\ 
\hline
RLfO & hopper & \textbf{106.47$\pm$1.06} & \textbf{105.65$\pm$0.18} & 47.01$\pm$5.73 & - & - & - & - & - & - & - & 103.39$\pm$1.66 \\ 
 & walker & \textbf{96.65$\pm$2.56} & \textbf{97.26$\pm$0.30} & 63.30$\pm$5.52 & - & - & - & - & - & - & - & 98.52$\pm$0.47 \\ 
 & halfcheetah & 50.38$\pm$1.51 & \textbf{68.66$\pm$0.70} & 3.35$\pm$0.42 & - & - & - & - & - & - & - & 63.57$\pm$0.98 \\ 
\hline
RL-expert & hopper & 87.73$\pm$7.30 & \textbf{105.56$\pm$0.28} & 51.54$\pm$4.30 & \textbf{98.63$\pm$0.15} & \textbf{98.08$\pm$0.27} & 65.45$\pm$3.05 & - & - & - & - &  103.39$\pm$1.66 \\ 
 & walker & \textbf{103.18$\pm$3.06} & \textbf{98.31$\pm$0.38} & 56.27$\pm$4.80 & 72.97$\pm$0.65 & 74.06$\pm$0.85 & 66.40$\pm$5.68 & - & - & - & - & 98.52$\pm$0.47 \\ 
 & halfcheetah & 48.43$\pm$2.02 & \textbf{64.37$\pm$1.50} & 3.47$\pm$0.83 & 13.68$\pm$2.58 & 9.99$\pm$2.69 & 3.41$\pm$0.83 & - & - & - & - & 63.57$\pm$0.98 \\ 
\hline
RL-sample & hopper & \textbf{103.08$\pm$1.80} & \textbf{101.66$\pm$3.24} & 71.92$\pm$10.30 & 0.95$\pm$0.01 & 0.94$\pm$0.00 & 71.50$\pm$7.80 & - & - & - & - &  103.34$\pm$2.55 \\ 
 & walker & \textbf{95.00$\pm$0.46} & \textbf{94.59$\pm$0.61} & 5.94$\pm$3.11 & -0.00$\pm$0.01 & 0.00$\pm$0.01 & 0.48$\pm$0.30 & - & - & - & - & 95.73$\pm$0.33 \\ 
 & halfcheetah & \textbf{68.30$\pm$0.76} & \textbf{68.71$\pm$0.92} & 13.26$\pm$4.07 & 20.36$\pm$3.95 & 24.58$\pm$4.49 & 19.38$\pm$4.18 & - & - & - & - & 69.91$\pm$0.43 \\
\bottomrule
\end{tabular}
\end{minipage}
\end{adjustbox}
\end{scriptsize}

\end{table}

%% file: tables/mw_results_with_std.tex
\begin{table}[t]

\begin{scriptsize}
\hspace{0.2\textwidth}
\begin{adjustbox}{angle=90}
\begin{minipage}[c][\textwidth][s]{\textheight}
\caption{Success rate of the final policy (with the exception of SMODICE$^*$) on Meta-World~\cite{yu2020meta}. The success rate is computed over 50 evaluation episodes. We consider an episode success if it is able to reach the goal within $128$ steps. We report the average success rate across 10 random seeds. MAHALO is one of the overall best-performing algorithms. ATAC also achieves strong results in IL and RL-expert because it is presented with expert-only data. $^*$ SMODICE is not able to train reliably, and often diverges during training. We therefore take the success rate of the best performing policy during training instead of the final one.
}
\label{tab:mw-rewards-std}
\vspace{5mm}
\begin{tabular}{c|c|c|c|c|c|c|c|c|c|c|c||c}
\toprule
Scenario & Dataset & MAHALO & RP & AP & UDS &  UDS-A & ATAC$^\dagger$ & BCO & BC & { SMODICE}$^*$ & SMODICE-$\chi^2$ & Oracle$^+$ \\ 
\hline
ILfO & reach & \textbf{65.0$\pm$3.1} & \textbf{62.6$\pm$1.8} & 13.0$\pm$2.2 & - & - & - & 11.6$\pm$1.7 & - & 10.6 $\pm$ 4.3 & 58.4$\pm$3.0 & - \\ 
 & push & \textbf{62.4$\pm$1.9} & 11.6$\pm$3.7 & 12.4$\pm$3.7 & - & - & - & 14.6$\pm$2.2 & - & 0.4$\pm$0.3 & \textbf{62.2$\pm$3.7} & - \\ 
 & plate-slide & \textbf{100.0$\pm$0.0} & 22.0$\pm$11.0 & \textbf{94.0$\pm$3.9} & - & - & - & 75.4$\pm$6.8 & -  & 0.0$\pm$0.0 & \textbf{99.8$\pm$0.2} & - \\ 
 & handle-press & 75.4$\pm$4.2 & 32.4$\pm$11.2 & \textbf{96.6$\pm$2.5} & - & - & - & {87.8$\pm$7.7} & - & 16.6$\pm$5.1 & \textbf{98.0$\pm$1.0} & - \\ 
 & button-press & \textbf{100.0$\pm$0.0} & \textbf{100.0$\pm$0.0} & \textbf{93.6$\pm$1.8} & - & - & - & \textbf{93.8$\pm$2.2} & - & 0.4$\pm$0.4 & \textbf{99.2$\pm$0.8} & - \\ 
\hline
IL & reach & 24.6$\pm$9.1 & 23.6$\pm$5.0 & 38.0$\pm$2.9 & 21.6$\pm$8.2 & 13.6$\pm$5.3 & \textbf{98.0$\pm$0.6} & - & 62.2$\pm$7.3 & 19.8$\pm$6.9 & 61.4$\pm$1.6 & - \\ 
 & push & \textbf{92.6$\pm$2.2} & 11.8$\pm$6.5 & 35.2$\pm$6.2 & 79.2$\pm$4.9 & 72.4$\pm$3.9 & 50.0$\pm$15.4 & - & \textbf{91.4$\pm$2.7} & 0.2$\pm$0.2 & 65.0$\pm$4.0 & - \\ 
 & plate-slide & 80.4$\pm$9.9 & 34.4$\pm$11.2 & {89.8$\pm$5.0} & 76.2$\pm$8.8 & 62.0$\pm$11.8 & {89.4$\pm$6.7} & - & {85.3$\pm$8.5} & 0.0$\pm$0.0 & \textbf{100.0$\pm$0.0}& - \\ 
 & handle-press & 71.4$\pm$5.7 & 34.8$\pm$9.6 & \textbf{100.0$\pm$0.0} & 35.2$\pm$10.0 & 57.4$\pm$11.2 & \textbf{97.0$\pm$1.8} & - & 75.2$\pm$5.7 & 20.2$\pm$5.7& \textbf{98.4$\pm$0.6} &  - \\ 
 & button-press & \textbf{100.0$\pm$0.0} & \textbf{99.8$\pm$0.2} & \textbf{96.2$\pm$2.8} & \textbf{100.0$\pm$0.0} & \textbf{100.0$\pm$0.0} & \textbf{99.6$\pm$0.4} & - & \textbf{100.0$\pm$0.0} & 0.2$\pm$0.2 & \textbf{100.0$\pm$0.0} & - \\ 
\hline
RLfO & reach & \textbf{86.4$\pm$1.8} & \textbf{88.0$\pm$2.4} & 15.0$\pm$2.0 & - & - & - & - & - & - & - & 51.6$\pm$13.2 \\ 
 & push & \textbf{58.2$\pm$3.4} & 32.0$\pm$2.2 & 20.2$\pm$3.1 & - & - & - & - & - & - & - & 91.8$\pm$1.6 \\ 
 & plate-slide & \textbf{100.0$\pm$0.0} & \textbf{100.0$\pm$0.0} & 83.2$\pm$5.4 & - & - & - & - & - & - & - & 100.0$\pm$0.0 \\ 
 & handle-press & 77.8$\pm$3.1 & 84.4$\pm$3.2 & \textbf{96.0$\pm$3.2} & - & - & - & - & - & - & - & 78.6$\pm$2.5 \\ 
 & button-press & \textbf{100.0$\pm$0.0} & \textbf{100.0$\pm$0.0} & \textbf{92.6$\pm$2.5} & - & - & - & - & - & - & - & 100.0$\pm$0.0 \\ 
\hline
RL-expert & reach & 39.2$\pm$12.1 & 54.0$\pm$13.0 & 42.0$\pm$2.3 & 57.8$\pm$12.3 & 33.6$\pm$10.8 & \textbf{98.4$\pm$0.8} & - & - & - & - & 51.6$\pm$13.2 \\ 
 & push & \textbf{95.6$\pm$1.7} & 88.6$\pm$1.4 & 40.4$\pm$5.5 & \textbf{90.4$\pm$1.9} & \textbf{90.0$\pm$2.2} & \textbf{99.6$\pm$0.4} & - & - & - & - & 91.8$\pm$1.6 \\ 
 & plate-slide & \textbf{100.0$\pm$0.0} & \textbf{100.0$\pm$0.0} & 89.6$\pm$3.8 & \textbf{99.4$\pm$0.6} & \textbf{99.0$\pm$0.6} & 85.4$\pm$7.3 & - & - & - & - & 100.0$\pm$0.0 \\ 
 & handle-press & 72.6$\pm$9.1 & 82.0$\pm$3.5 & \textbf{97.6$\pm$2.1} & 81.2$\pm$2.6 & 80.8$\pm$2.9 & \textbf{99.0$\pm$0.4} & - & - & - & - & 78.6$\pm$2.5 \\ 
 & button-press & \textbf{100.0$\pm$0.0} & \textbf{100.0$\pm$0.0} & \textbf{97.0$\pm$1.7} & \textbf{100.0$\pm$0.0} & \textbf{100.0$\pm$0.0} & \textbf{96.4$\pm$3.4} & - & - & - & - & 100.0$\pm$0.0 \\ 
\hline
RL-sample & reach & \textbf{86.6$\pm$1.4} & \textbf{87.4$\pm$2.4} & 63.0$\pm$3.0 & \textbf{87.4$\pm$1.6} & \textbf{89.8$\pm$1.8} & \textbf{85.4$\pm$3.0} & - & - & - & - & 88.8$\pm$1.8 \\ 
 & push & 40.6$\pm$3.8 & \textbf{47.8$\pm$4.0} & 29.2$\pm$4.9 & 35.8$\pm$3.3 & 37.4$\pm$4.4 & 35.0$\pm$4.8 & - & - & - & - & 46.0$\pm$2.3 \\ 
 & plate-slide & \textbf{100.0$\pm$0.0} & \textbf{100.0$\pm$0.0} & \textbf{97.6$\pm$1.2} & \textbf{100.0$\pm$0.0} & \textbf{100.0$\pm$0.0} & \textbf{99.6$\pm$0.4} & - & - & - & - & 100.0$\pm$0.0 \\ 
 & handle-press & 78.4$\pm$3.9 & 81.0$\pm$2.6 & \textbf{100.0$\pm$0.0} & 76.8$\pm$6.6 & 77.8$\pm$6.2 & 82.8$\pm$4.4 & - & - & - & - &  83.4$\pm$3.1 \\ 
 & button-press & \textbf{100.0$\pm$0.0} & \textbf{100.0$\pm$0.0} & \textbf{100.0$\pm$0.0} & \textbf{100.0$\pm$0.0} & \textbf{100.0$\pm$0.0} & \textbf{100.0$\pm$0.0} & - & - & - & - & 100.0$\pm$0.0 \\ 
 \bottomrule
\end{tabular}
\end{minipage}
\end{adjustbox}
\end{scriptsize}
\vskip -0.1in
\end{table}